\documentclass[sigconf, anonymous=false, review=false]{acmart}

\usepackage{amsfonts}       
\usepackage{amssymb}
\usepackage{algorithm,algorithmicx,algpseudocode}
\usepackage{algcompatible}
\usepackage{amsmath}
\newtheorem{assumption}{Assumption}

\usepackage{subcaption}
\usepackage{booktabs}
\usepackage{epsfig}
\usepackage{epstopdf}
\usepackage{url}
\usepackage{multirow}
\begin{document}

\title{Faster On-Device Training Using New Federated \\Momentum Algorithm}

\lhead[]{Faster On-Device Training Using New Federated Momentum Algorithm}
\author{Zhouyuan Huo$^1$, Qian Yang$^2$, Bin Gu$^1$, Lawrence Carin$^2$, Heng Huang$^1$}
\email{zhouyuan.huo@pitt.edu}
\affiliation{%
  \institution{$^1$  University of Pittsburgh, $^2$  Duke University}
}


\begin{abstract}
    Mobile crowdsensing has gained significant attention in recent years and has become a critical paradigm for emerging Internet of Things applications. The sensing devices continuously generate a significant quantity of data, which provide tremendous opportunities to develop innovative intelligent applications. To utilize these data to train machine learning models while not compromising user privacy, federated learning has become a promising solution. However, there is little understanding of whether federated learning algorithms are guaranteed to converge. We reconsider model averaging in federated learning and formulate it as a gradient-based method with biased gradients. This novel perspective assists analysis of its convergence rate and provides a new direction for more acceleration. We prove for the first time that the federated averaging algorithm is guaranteed to converge for non-convex problems, without imposing additional assumptions. We further propose a novel accelerated federated learning algorithm and provide a convergence guarantee.  Simulated federated learning experiments are conducted to train deep neural networks on benchmark datasets, and experimental results show that our proposed method converges faster than previous approaches.
\end{abstract}

\maketitle

\section{Introduction}

Mobile crowdsensing is a new paradigm of sensing by taking advantage of the power of various mobile devices, which are penetrating most aspects of modern life and also continuously generate a large amount of data. As new high-speed 5G networks arrive to handle their traffic, the number of connected smart devices is expected to grow further over the next five years \cite{5g}. It is desirable to utilize these data to improve model performance and maximize the user experience. Traditional distributed optimization methods are able to train models when all datasets are stored in the cluster \cite{dean2012large}.  However, the increasing awareness of user privacy and data security issues prevents storing and training a model on a centralized server \cite{yang2019federated}.  It therefore becomes a major challenge to train a model with massive distributed and heterogeneous datasets without compromising user privacy.

\begin{figure}[!tb]
	\centering
	\includegraphics[width=3.in]{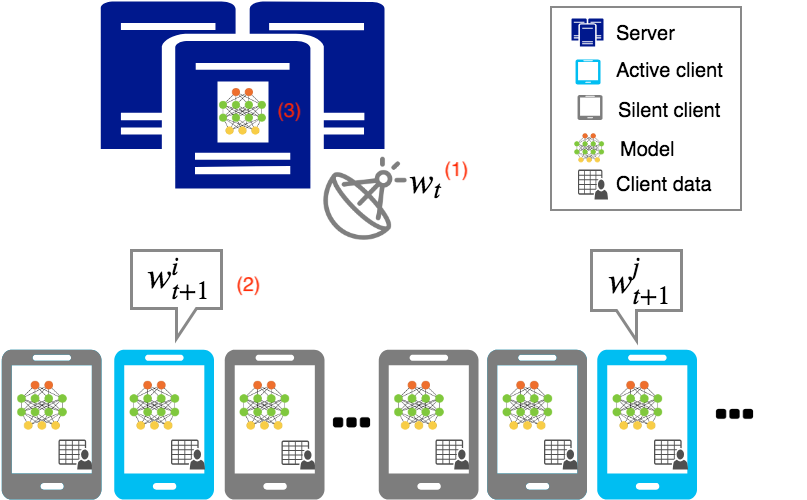}
	\caption{Federated learning procedure: (1) Server selects a set of Active clients and broadcasts model $w_t$. (2) After receiving model from Server, Active clients conduct the update locally and send the updated model ($e.g.$, $w_{t+1}^i$ and $w_{t+1}^j$ in the figure) back to Server. 3) Server updates the model $w_{t+1}$.  Steps 1-3 are repeated until convergence.}
	\label{intro}
\end{figure}

Federated learning \cite{mcmahan2016communication,yang2019federated,kairouz2019advances} is a promising solution for machine learning model training using crowdsensing data without compromising user privacy. It has been widely applied for mobile keyboard prediction \cite{hard2018federated}, private language models \cite{mcmahan2017learning}, and financial-client classification \cite{cheng2019secureboost}. As shown in Figure \ref{intro}, the data from each client are never uploaded to the server during the optimization period. Instead, each client conducts updates locally for several iterations and sends the updated model back to the server. At each iteration, the server only has access to a  small fraction of clients, and updates the server model after receiving from these active clients. The Google Gboard team trains a recurrent neural network (RNN) with $1.4$ million parameters for the next-word prediction task \cite{hard2018federated}, where a round of training takes $2$ to $3$ minutes. It requires roughly $3000$ rounds to converge, with a total running time of over $5$ days -- taking much more wall-clock time than centralized training on the same task \cite{dean2012large}. It is nontrivial to accelerate the training of federated learning.
Furthermore, \cite{bonawitz2019towards} reports a large difference in the number of participating devices over a 24-hour period for a US-centric client population, which consequently has an impact on the round-completion rate. Because the communication between server and clients is unstable and expensive,  \cite{konevcny2016federated,caldas2018expanding,jeong2018communication} proposed new techniques to reduce communication and accelerate federated-learning training. Nonetheless, few studies provide a convergence guarantee of federated learning algorithms, especially for non-convex problems. There are two difficulties that make this a challenging problem: the data are non-independent identical distribution (non-IID) and there is limited communication between server and clients.

Synchronous and asynchronous gradient-based methods have already been proven to converge to critical points for non-convex problems \cite{li2014efficient,lian2015asynchronous,reddi2015variance}.  To reduce communication overhead, local stochastic gradient descent (local SGD) is studied in the field of distributed optimization for clusters \cite{stich2018local,lin2018don,yu2018parallel}. In \cite{lin2018don}, the authors showed that local SGD converges faster than mini-batch gradient descent with fewer communications. \cite{stich2018local,zhou2017convergence,yu2018parallel} investigated the convergence of local SGD and proved that it is guaranteed to converge for strongly convex or non-convex problems. Another line of communication-efficient distributed methods \cite{lee2015distributed,zhang2015disco,ma2015adding} performed model averaging after solving local subproblems and are guaranteed to converge. However, the above distributed methods either assume that the dataset is distributed IID or requires the server to collect updates from all workers at each iteration. None of these approaches is applicable to federated learning.

We investigate the model averaging in federated learning from a new point of view and formulate it as a gradient-based method with biased gradients. This novel perspective helps the analysis of its convergence rate and motivates a new accelerated method. Our main contributions in this paper are summarized as follows:
\begin{itemize}
	\item We investigate the model averaging step of the FedAvg algorithm \cite{mcmahan2016communication} and derive the first convergence proof for non-convex problems;
	\item A novel algorithm is proposed to accelerate federated optimization, and it also provides a convergence analysis for non-convex problems;
	\item  We perform simulated federated learning experiments with training deep neural networks, and the empirical results show that the proposed method converges faster than previous approaches.
\end{itemize}

\section{Related Works}
\textbf{Distributed Learning.}
When a dataset $\mathcal{P}$ of $n$ samples is centrally stored and partitioned across $K$ machines or clients $\mathcal{P} = \{\mathcal{P}_1,\mathcal{P}_2,\cdots,\mathcal{P}_K\}$, we use distributed optimization methods to train machine learning models.  $\mathcal{P}_k$ represents the set of data indices on worker $k$, and we let  $|\mathcal{P}_k| = n_k$ and $n = \sum_{k=1}^K n_k$. The target of  distributed learning is to minimize a weighted finite-sum loss of all clients as follows:
\begin{eqnarray}
	\min\limits_{w \in \mathbb{R}^d}  \hspace{0.1cm}  \left\{ f(w) \hspace{0.1cm} :=  \hspace{0.1cm}    \sum\limits_{k=1}^K \frac{n_k}{n} f_k(w)\right\}  ,
	\label{obj}
\end{eqnarray}
where $f(w)$ and $f_k(w)$ are non-convex problems. $f_k(w)$ denotes a subset of loss on client $k$ and  $f_k(w) :=  \mathbb{E}_{\xi \in \mathcal{P}_k} [f_i(w, \xi)]$, where $\xi$ is a sampled data from $\mathcal{P}_k$.  If the dataset is randomly partitioned and $n_k = {n}/{K}$, $f(w)$ can also be represented as $\frac{1}{K} \sum_{k=1}^K f_k(w)$. Distributed mini-batch gradient methods have been used widely to train deep neural networks \cite{dean2012large,chen2016revisiting,you2017large}. At each iteration, gradients are computed on clients and aggregated on the server. However, this method suffers severely from network delays and bandwidth limits. To overcome the communication bottleneck, \cite{zhang2015deep,lin2018don} allowed workers to perform local updates for a while and average local models on the server periodically. In \cite{stich2018local}, the authors analyzed the convergence of local SGD for convex problems. Recently,  \cite{zhou2017convergence,yu2018parallel,wang2018cooperative} applied local SGD to non-convex problems and proved that it guarantees convergence as well.  However, local SGD requires averaging local models from all clients, which is not realistic for federated learning.

\textbf{Federated Learning.} Instead of training machine learning models with centrally stored data, as with distributed learning, federated learning seeks to perform large-scale learning on a massive number of distributed mobile devices with  heterogeneous data \cite{mcmahan2016communication,smith2017federated}. We summarize the differences between distributed and federated learning in Table \ref{table:fed}. There are two challenges in federated learning: ($i$) the datasets are massive, distributed, and heterogeneous, so that $K$ is extremely large and $n_k$ is highly unbalanced in problem (\ref{obj}); ($ii$) communications between server and clients are highly unstable, so that only a small fraction of clients are active within one iteration.  To improve federated learning empirically, \cite{konevcny2016federated} investigated reducing the uplink (clients $\rightarrow$ server) communication cost through message compression. Later, \cite{caldas2018expanding} proposed to reduce the downlink (server $\rightarrow$ clients) communication cost by training smaller sub-models. However, few studies have analyzed the convergence of federated learning methods. \cite{li2019convergence} analyzed the convergence of federated averaging algorithm for convex problems.  Recently, \cite{sahu2018convergence} proved that federated learning is guaranteed to converge for non-convex problems. However, the authors considered a different federated learning algorithm,  imposing a new regularization and considered unrealistic assumptions, such as strongly convex subproblems and bounded dissimilarity of local data.


\begin{table}[!t]
	\center
	\caption{Comparisons of settings between federated learning and normal distributed learning. IID denotes independent and identically distributed.}
	\setlength{\tabcolsep}{3mm}
	\begin{tabular}{ccc}
		\hline						
		&  \textbf{Distributed}  & \textbf{Federated }   \\
		\hline
		\multirow{2}{*}{\textbf{Data}}	& IID   & Non-IID \\
													& Balanced  &  Unbalanced\\
		\hline
		\multirow{2}{*}{\textbf{\small Commun-}}	& Centralized cluster & { Massively distributed  } \\
		\textbf{-ication} &   Stable & {   Limited } \\
		\hline						
	\end{tabular}
	\label{table:fed}
\end{table}

\section{Federated Averaging Algorithm}
\label{sec::fedavg}
We first briefly introduce the Federated Averaging algorithm (FedAvg) proposed in \cite{mcmahan2016communication}. We then reformulate the model averaging in FedAvg as a gradient-based method and prove that it is guaranteed to converge to critical solutions for non-convex problems.

\subsection{Algorithm Description}
According to the setting of federated learning, we assume there is one server and $K$ clients, where $K$ is a large value. Algorithm \ref{alg_server} summarizes the procedures of FedAvg on the server.  At iteration $t$, the server is connected with a set of  active clients $S_t$ with the size $M$, where $M \ll K$. After broadcasting model $w_t$ to active clients $S_t$, the server waits until receiving updated model $w_{t+1}^k$ through Algorithm \ref{alg_worker} from $S_t$. Finally, model  $w_t$  is updated on the server via model averaging:
\begin{eqnarray}
	w_{t+1} &= & \sum\limits_{k=1}^K \frac{n_k}{n} w_{t+1}^k \hspace{0.1cm} =  \hspace{0.1cm}  \sum\limits_{k \in S_t} \frac{n_k}{n} w_{t+1}^k + \sum\limits_{k \notin S_t} \frac{n_k}{n} w_{t}.
	\label{iq_avg}
\end{eqnarray}
Note that $w^{t+1}$ is averaging local updated models from {\em all} clients in (\ref{iq_avg}) by setting $w_{t+1}^k=w_t$ for any $k\notin S_t$,  which is consistent with the FedAvg algorithm in the original paper \cite{mcmahan2016communication}. The right term in (\ref{iq_avg}) enforces that $w_{t+1}$ stay close to the current server model $w_t$ implicitly if $M\ll K$.  In \cite{sahu2018convergence}, the authors updated the server model  using local models from only active clients through $w_{t+1} = \frac{1}{M} \sum_{k \in S_t} w_{t+1}^k$. Because of this step, they have to impose new regularization on local subproblems and bring in an additional assumption about data distributions for convergence analysis.

Algorithm \ref{alg_worker}  describes the local solver on clients.  After receiving model $w_t$ from the server, client $k$ applies SGD locally and updates $w_h^k$ iteratively for $H_t$ iterations. Finally, it sends the updated local model $w_{t+1}^k $ back to the server. The local solver can also be any gradient-based method, such as Momentum method \cite{qian1999momentum,yang2016unified}, RMSProp \cite{hinton2012neural}, Adam \cite{kingma2014adam} or AdamW \cite{loshchilov2017fixing}. We only consider SGD in this paper, for simplicity.

\begin{figure}[t]
	\begin{minipage}[t]{0.48\textwidth}
		\begin{algorithm}[H]
			\renewcommand{\algorithmicrequire}{\textbf{Phase 2:}}
			\renewcommand{\algorithmicensure}{\textbf{Initialize:}}
			\caption{Federated Averaging Algorithm  \hfill [Server]}
			\begin{algorithmic}[1]
				\ENSURE  $w_0$; 
				\FOR{$t=0,1,\cdots,T-1$}
				\STATE  $S_{t} =$  a randomly sampled set of $M$ clients $(M\ll K)$ ;
				\FOR{each client $k \in S_{t}$ \textbf{in parallel}}
				\STATE Send $w_t$ to client $k$;
				\STATE  Receive $w^k_{t+1} $ from client $k$ via Algorithm \ref{alg_worker};
				\ENDFOR	
				\STATE $w_{t+1}^k = w_t$ for any $k \notin S_{t}$;
				\STATE Update $w_{t+1} =  \sum_{k=1}^K \frac{n_k}{n} w_{t+1}^k.$
				\ENDFOR	
			\end{algorithmic}
			\label{alg_server}		
		\end{algorithm}
	\end{minipage}%
	\hspace{0.3cm}
	\begin{minipage}[t]{0.48\textwidth}
		\begin{algorithm}[H]
			\renewcommand{\algorithmicrequire}{\textbf{Phase 2:}}
			\renewcommand{\algorithmicensure}{\textbf{Initialize:}}
			\caption{Federated Averaging Algorithm  \hfill [Client $k$]}
			\begin{algorithmic}[1]
				\STATE Receive  $w_t$ from the server;
				\STATE Initialize $w^k_{t, 0} = w_t$;
				\FOR{$h=0,1,\cdots,H_t-1$}
				\STATE Select sample $\xi$ randomly from $\mathcal{P}_k$;
				\STATE Update $w^k_{t, h+1} = w^k_{t, h} - \gamma_t \nabla f_{k}(w^k_{t,h}, \xi)$;
				\ENDFOR	
				\STATE Set $w_{t+1}^k = w^k_{t, H_t}$;
				\STATE Send $w_{t+1}^k$  back to the server.
			\end{algorithmic}
			\label{alg_worker}		
		\end{algorithm}
	\end{minipage}
\end{figure}

\subsection{Model Averaging Is a Gradient-Based Method with Biased Gradients}
Clients in federated learning are mobile devices, and the server is a cluster and able to do more computations  not limited to model averaging. In this paper, we reconsider the model averaging at Line 8  in Algorithm \ref{alg_server}, and formulate it as an update of gradient-based methods as follows:
\begin{eqnarray}
	w_{t+1} \hspace{0.1cm}=\hspace{0.1cm}  w_{t} -  g_t, \hspace{0.5cm}  \text{ where } \hspace{0.2cm}   g_t  \hspace{0.1cm}= \hspace{0.1cm} \sum\limits_{k=1}^K  \frac{n_k}{ n }\left( w_t - w_{t+1}^k  \right).
	\label{iq_1_0000}
\end{eqnarray}
$g_t$ denotes a biased gradient on the server at iteration $t$, because $\mathbb{E}_{S_t}[g_t] \neq \nabla f(w_t)$. Equations   (\ref{iq_avg}) and (\ref{iq_1_0000}) are equivalent.   Given that $w_{t+1}^k = w_t$ for any $k \notin S_t$,  we can also rewrite $g_t= \sum_{k \in S_t}  \frac{n_k}{ n }  \sum_{h=0}^{H_t-1} \nabla f_{k}(w^k_{t,h}, \xi)$  in (\ref{iq_1_0000}) using gradients computed on clients in Algorithm \ref{alg_worker}.
To generalize (\ref{iq_1_0000}), we set a constant learning rate $\eta  \in \left[ 1,  \frac{K}{M}\right]$ and  obtain the following update function in the server:
\begin{eqnarray}
	w_{t+1} &= &  w_{t} - \eta  \sum_{k \in S_t}  \frac{n_k}{ n }  \sum_{h=0}^{H_t-1} \nabla f_{k}(w^k_{t,h}, \xi). \nonumber
\end{eqnarray}
If $\eta=1$, it is equivalent to (\ref{iq_1_0000}).
By rewriting model averaging as a gradient-based method, we can easily utilize existing theoretical analysis and many improved algorithms of gradient-based methods. In the following context, we provide the convergence analysis of FedAvg for non-convex problems, and propose a novel accelerated federated learning algorithm.

\textbf{Difficulty of the Convergence Analysis:}  It is difficult to analyze the convergence of FedAvg, due to its biased gradient.
Conventional analysis for gradient-based methods requires $\mathbb{E} [g_t] = \nabla f(w_t)$. However, this is not satisfied in FedAvg. 
Apart from the unbiased gradients, limited communication in federated learning also increases the difficulty of analysis.

\subsection{Convergence Analysis}
To prove the convergence rate of FedAvg, we assume that two widely used assumptions \cite{bottou2016optimization}, Bounded Variance and Lipschitz Continuous Gradient, are satisfied throughout this paper.
\begin{assumption}[Bounded Variance]
	We assume that the variance of stochastic gradient on local clients is upper bounded, so that for any $w\in \mathbb{R}^d$ and $k \in \{1,...,K\}$, it is satisfied that:
	\begin{eqnarray}
		\mathbb{E}_{\xi \sim \mathcal{P}_k} \left\|  \nabla f_k(w, \xi) -   \nabla f_k(w) \right\|_2^2   &\leq &   \sigma^2. \nonumber
	\end{eqnarray}
	\label{ass_bd}
\end{assumption}

\begin{assumption}[Lipschitz Continuous Gradient)]
	The gradients of $f$ and $f_k$ are Lipschitz continuous with a constant $L>0$, so that for any $w, v\in\mathbb{R}^d$ and $k \in \{1,...,K\}$, it is satisfied that:
	\begin{eqnarray}
		\|\nabla f_k(w) - \nabla f_k(v) \|_2 &\leq& L  \| w -  v \|_2.  \nonumber \\
				\|\nabla f(w) - \nabla f(v) \|_2 &\leq& L  \| w -  v \|_2. \nonumber
	\end{eqnarray}
	\label{ass_lips}
\end{assumption}
Under Assumptions \ref{ass_bd} and \ref{ass_lips}, we can prove the upper bound of FedAvg at each iteration.
\begin{lemma}
	Under Assumptions \ref{ass_bd} and \ref{ass_lips}, the update of $w_t$ on the server at each iteration is upper bounded as follows:
	\begin{eqnarray}
	&	\mathbb{E}_{\xi,k} f(w_{t+1}) \hspace{0.2cm} \leq  \hspace{0.2cm}  f(w_t ) -  \frac{ MH_t \eta \gamma_t}{2K}  \|  \nabla f(w_t) \|_2^2 \nonumber \\
	&+ \left( \frac{M L H_t \eta^2 \gamma_t^2}{2K} + \frac{M  L^2 \eta H_t^2\gamma_t^3}{2K}    \right) \sigma^2 	- \bigg( \frac{M \eta \gamma_t }{2K}  \nonumber \\
	& -  \frac{ML^2 \eta H_t^2\gamma_t^3}{2K} - \frac{M L H_t \eta^2 \gamma_t^2}{2K}  \bigg) \sum\limits_{h=0}^{H_t-1} \sum\limits_{k=1}^K  \frac{n_k}{n}  \mathbb{E}_{\xi, k} \| \nabla f_k (w_{t, h}^k) \|_2^2.
	\label{iq_301}
	\end{eqnarray}
	\label{lem1}
\end{lemma}

\begin{proof}
	According to  Assumption \ref{ass_lips}, it holds that:
\begin{eqnarray}
	 \mathbb{E}_{\xi} 	f(w_{t+1})	&\leq & f(w_t) +\mathbb{E}_{\xi}  \left< \nabla f(w_t) , w_{t+1} - w_t \right> \nonumber \\
	 &&+ \frac{L}{2}\mathbb{E}_{\xi}  \left\| w_{t+1} - w_t \right\|_2^2 \nonumber
\end{eqnarray}	
 Taking an expectation over samples on the inequality above, we have:
	{\small
		\begin{eqnarray}
		&& \mathbb{E}_{\xi} 	f(w_{t+1}) \nonumber \\
		&\leq & f(w_t) +\mathbb{E}_{\xi}  \left< \nabla f(w_t) , w_{t+1} - w_t \right> + \frac{L}{2}\mathbb{E}_{\xi}  \left\| w_{t+1} - w_t \right\|_2^2 \nonumber \\
		&\leq &  f(w_t ) - \mathbb{E}_{ \xi}  \left< \nabla f(w_t) ,  \eta \gamma_t\sum\limits_{k \in S_t} \frac{n_k}{n} \sum\limits_{h=0}^{H_t-1}   \nabla f_k (w_{t, h}^k, \xi) \right>  \nonumber \\
		&&	 + \frac{L}{2}\mathbb{E}_{ \xi}   \left\| \eta\gamma_t \sum\limits_{k \in S_t} \frac{n_k}{n}  \sum\limits_{h=0}^{H_t-1}  \nabla f_k (w_{t, h}^k, \xi) \right\|_2^2 \nonumber \\
		&= &  f(w_t ) -   \left< \nabla f(w_t) ,  \eta\gamma_t \sum\limits_{k \in S_t} \frac{n_k}{n} \sum\limits_{h=0}^{H_t-1}   \nabla f_k (w_{t, h}^k) \right>  \nonumber \\
		&&+ \frac{L}{2}\mathbb{E}_{ \xi}   \left\|\eta \gamma_t \sum\limits_{k \in S_t} \frac{n_k}{n}  \sum\limits_{h=0}^{H_t-1}  \nabla f_k (w_{t, h}^k, \xi) \right\|_2^2,
		\label{iq_20001}
		\end{eqnarray}}
	where the inequality follows from $\mathbb{E}_{\xi}[\nabla f_k(w, \xi)] = \nabla f_k(w) $.
	Because all the clients are selected randomly with a uniform distribution, we take expectation over clients and  have:
	\begin{eqnarray}
	&& \mathbb{E}_{\xi, k} f(w_{t+1}) \nonumber\\
	& \leq &  f(w_t ) - \frac{M\eta\gamma_t }{K}  \sum\limits_{h=0}^{H_t-1}   \sum\limits_{k=1}^K  \frac{n_k}{n}   \mathbb{E}_{\xi, k}  \left< \nabla f_k(w_t) ,  \nabla f_k (w_{t, h}^k) \right>  \nonumber \\
	&& + \frac{L}{2K} \sum\limits_{k=1}^K   \mathbb{E}_{\xi, k}    \left\| \eta \gamma_t \sum\limits_{k \in S_t} \frac{n_k}{n} \sum\limits_{h=0}^{H_t-1}  \nabla f_k (w_{t, h}^k, \xi) \right\|_2^2 \nonumber \\
	&\leq &     f(w_t ) - \frac{M\eta\gamma_t }{K}  \sum\limits_{h=0}^{H_t-1}   \sum\limits_{k=1}^K  \frac{n_k}{n} \mathbb{E}_{\xi, k}   \left< \nabla f_k(w_t) ,  \nabla f_k (w_{t, h}^k) \right>  \nonumber \\
	&& + \frac{M  L \eta^2\gamma_t^2}{2K}    \sum\limits_{k=1}^K  \frac{n_k}{n}  \mathbb{E}_{\xi, k}  \left\|   \sum\limits_{h=0}^{H_t-1}  \nabla f_k (w_{t, h}^k, \xi) \right\|_2^2 \nonumber \\
	&= &    f(w_t ) -  \frac{MH_t\eta \gamma_t }{2K}  \sum\limits_{k=1}^K  \frac{n_k}{n}    \mathbb{E}_{\xi, k}  \|  \nabla f_k(w_t) \|_2^2  \nonumber \\
	&& - \frac{M\eta\gamma_t }{2K} \sum\limits_{h=0}^{H_t-1} \sum\limits_{k=1}^K  \frac{n_k}{n}  \mathbb{E}_{\xi, k}  \left\|    \nabla f_k (w_{t, h}^k) \right\|_2^2 \nonumber \\
	&& +\frac{M \eta\gamma_t }{2K} \sum\limits_{h=0}^{H_t-1}\sum\limits_{k=1}^K  \frac{n_k}{n} \underbrace{ \mathbb{E}_{\xi, k} \left\|  \nabla f_k(w_t) -    \nabla f_k (w_{t, h}^k) \right\|_2^2}_{Q_1}  \nonumber \\
	&&  + \frac{M L\eta^2\gamma_t^2}{2K}   \sum\limits_{k=1}^K  \frac{n_k}{n}     \underbrace{ \mathbb{E}_{\xi, k}   \left\| \sum\limits_{h=0}^{H_t-1 }   \nabla f_k (w_{t, h}^k, \xi) \right\|_2^2 }_{Q_2},
	\label{iq_10001}
	\end{eqnarray}
	where the first inequality follows from  $  \mathbb{E}_k \sum\limits_{k \in S_t} \frac{n_k}{n}  \nabla f_k (w_{t, h}^k) =  \frac{M}{K}\sum\limits_{k=1}^K  \frac{n_k}{n}  \nabla f_k (w_{t, h}^k) $ and $\sum\limits_{k=1}^K  \frac{n_k}{n} \nabla f_k(w_t) = \nabla f(w_t)$, the second inequality follows from Jensen's inequality,  and the last equality follows from $\left<a, b\right> = \frac{1}{2} (\|a\|_2^2 + \|b\|_2^2- \|a-b\|_2^2 )  $. We prove the upper bound of $Q_1$ as follows:
	\begin{eqnarray}
	Q_1  &=&  \mathbb{E}_{\xi, k}  \|  \nabla f_k(w_t)  -    \nabla f_k (w_{t, h}^k) \|_2^2 \nonumber \\
	&\leq & L^2  \mathbb{E}_{\xi, k}  \left\| \sum\limits_{j=0}^{h-1} \left( w^k_{t,j+1}   -    w_{t, j}^k \right) \right\|_2^2 \nonumber \\
	&= & L^2 \gamma_t^2   \mathbb{E}_{\xi, k}  \left\|    \sum\limits_{j=0}^{h-1}\left(  \nabla f_k(w_{t, j}^k, \xi)  -  \nabla f_k(w^k_{t,j}) +  \nabla f_k(w^k_{t,j} )  \right) \right\|_2^2, \nonumber \\
	& = & L^2 \gamma_t^2   \mathbb{E}_{\xi, k}  \left\|    \sum\limits_{j=0}^{h-1}  \left( \nabla f_k(w_{t, j}^k, \xi)  -  \nabla f_k(w^k_{t,j}) \right) \right\|_2^2 \nonumber \\
	&& + L^2 \gamma_t^2   \mathbb{E}_{\xi, k}  \left\| \sum\limits_{j=0}^{h-1}  \nabla f_k(w^k_{t,j} )  \right\|_2^2, \nonumber \\
	&\leq & h L^2 \gamma_t^2 \sigma^2 + L^2 \gamma_t^2 h \sum\limits_{j=0}^{h-1}   \mathbb{E}_{\xi, k}  \left\|   \nabla f_k(w^k_{t,j} )  \right\|_2^2, \nonumber
	\end{eqnarray}
	where the first inequality follows from Assumption \ref{ass_lips} and $w_t = w_{t,0}^k$, the second inequality follows from $\|z_1 + ... + z_n \|_2^2 \leq  n ( \|z_1\|_2^2 + ... + \|z_n\|_2^2)$	for any $z_1$, ... , $z_n$ and $\mathbb{E} \|z_1 + ... + z_n \|_2^2 \leq \mathbb{E} \left[   \|z_1\|_2^2 + ... + \|z_n\|_2^2\right]$	for any random variable $z_1$, ... , $z_n$ with mean $0$.
 The last inequality follows from Assumption \ref{ass_bd}.  Summing the inequality above from $h=0$ to $H_t-1$, we know that:
	\begin{eqnarray}
	\sum\limits_{h=0}^{H_t-1}  Q_1 &\leq & L^2 H_t^2 \gamma_t^2 \sigma^2  + L^2 \gamma_t^2H_t^2  \sum\limits_{h=0}^{H_t-1}  \mathbb{E}_{\xi, k}  \left\| \nabla f_k(w_{t, h}^k) \right\|_2^2. \nonumber
	\end{eqnarray}
	To get the upper bound of $Q_2$, we have:
	\begin{eqnarray}
	Q_2 &=  &  \mathbb{E}_{\xi, k}    \left\|  \sum\limits_{h=0}^{H_t-1}  \left( \nabla f_k (w_{t, h}^k, \xi)  - \nabla f_k (w_{t, h}^k) + \nabla f_k (w_{t, h}^k) \right)  \right\|_2^2 \nonumber \\
	& =    & \mathbb{E}_{\xi, k}    \left\|\sum\limits_{h=0}^{H_t-1}   \left(  \nabla f_k (w_{t, h}^k, \xi)  - \nabla f_k (w_{t, h}^k) \right) \right\|_2^2  \nonumber \\
	&& +\mathbb{E}_{\xi, k}  \left\| \sum\limits_{h=0}^{H_t-1} \nabla f_k (w_{t, h}^k)   \right\|_2^2  \nonumber \\
	&  \leq   & H_t  \sigma^2  + H_t \sum\limits_{h=0}^{H_t-1} \mathbb{E}_{\xi, k} \left\| \nabla f_k (w_{t, h}^k)\right\|_2^2,
	\label{iq_601}
	\end{eqnarray}
	where the second equality follows from $\mathbb{E}_{\xi} [ \nabla f_k (w_{t, h}^k, \xi)]  = \nabla f_k (w_{t, h}^k) $ and the inequality is from Assumption \ref{ass_bd}. Inputting $Q_1$ and $Q_2$  into inequality (\ref{iq_10001}), we complete the proof.
	
\end{proof}

From Lemma \ref{lem1}, we can ensure the convergence of $f$ at each iteration as long as $\gamma_t$ is properly selected and $ 1 -  {L^2 H_t^2\gamma_t^2} - { L H_t\eta \gamma_t}   \geq 0$. According to Lemma \ref{lem1}, we can prove the convergence of FedAvg for non-convex problems as follows.
\begin{theorem}
	\label{thm1}
	Assume that Assumptions  \ref{ass_bd} and \ref{ass_lips} hold. We let local iteration $H_t=H$,  $\eta \in \left[1, \frac{K}{M} \right]$, and stepsize sequence $\{\gamma_t\}$ satisfies $\gamma_t \leq \min \left\{ \frac{1}{2LH_t}, \frac{1}{4\eta LH_t} \right\}$ for all $t \in \{0,...,T-1 \}$. In addition, we assume loss $f$ has a lower bound $f_{\inf}$ and let $\Gamma_{T} = \sum_{t=0}^{T-1} { \gamma_t } $. Then, the output of Algorithm \ref{alg_server} satisfies that:
	\begin{eqnarray}
		\min\limits_{t \in \{0,...,T-1\}}  \mathbb{E}_{\xi, k}  \|  \nabla f(w_t) \|_2^2 & \leq & \frac{2K \left(f(w_0) - f_{\inf}  \right)}{ \eta M H \Gamma_{T}} \nonumber \\
		&&+    \frac{ L    \sigma^2 (2\eta + 1) }{2}  \frac{\sum\limits_{t=0}^{T-1}  \gamma_t^2}{\Gamma_{T}}. \nonumber
	\end{eqnarray}
\end{theorem}
\begin{proof}
	From Lemma \ref{lem1}, if we let $\gamma_t \leq \min \left\{ \frac{1}{2LH_t}, \frac{1}{4\eta LH_t} \right\}$, we have:
	\begin{eqnarray}
	\frac{M \eta \gamma_t }{2K}  -  \frac{ML^2 \eta H_t^2\gamma_t^3}{2K} - \frac{M L  H_t \eta^2 \gamma_t^2}{2K} &> & 0. \nonumber
	\end{eqnarray}
	Taking expectation of inequality (\ref{iq_301}), it holds that:
	\begin{eqnarray}
	\mathbb{E}_{\xi, k} f(w_{t+1}) &\leq & \mathbb{E}_{\xi, k}  f(w_t ) -  \frac{M H_t \eta \gamma_t}{2K} \mathbb{E}_{\xi, k} \|  \nabla f(w_t) \|_2^2 \nonumber \\
	&&  + \frac{M L H_t  \gamma_t^2\sigma^2}{4K} \left( 2 \eta^2+ \eta    \right). \nonumber
	\end{eqnarray}
	Supposing $H_t = H$, rearranging the inequality above, and summing it up from $t=0$ to $T-1$, we have:
	{
		\begin{eqnarray}
		\sum\limits_{t=0}^{T-1} { \gamma_t }  \mathbb{E}_{\xi, k}  \|  \nabla f(w_t) \|_2^2 &\leq&  \frac{2K \left(f(w_0) - f(w_T)  \right)}{\eta MH} \nonumber \\
		&&+  \frac{L  \sigma^2 \left( 2 \eta+ 1    \right)  }{2}   \sum\limits_{t=0}^{T-1}  \gamma_t^2.\nonumber
		\end{eqnarray}}
	Following \cite{bottou2016optimization}, let $\Gamma_{T} = \sum\limits_{t=0}^{T-1} {\gamma_t } $ and $f_{\inf }\leq f(w_T)$, we have:
	{
		\begin{eqnarray}
		\min\limits_{t \in \{0,...,T-1\}}    \|  \nabla f(w_t) \|_2^2  &\leq &  \frac{2K \left(f(w_0) - f(w_{\inf})  \right)}{ \eta MH \Gamma_{T}}  \nonumber \\
		&&+    \frac{L  \sigma^2 \left( 2 \eta+ 1    \right)  }{2} \frac{\sum\limits_{t=0}^{T-1}   \gamma_t^2}{\Gamma_{T}}.  \nonumber
		\end{eqnarray}}

	We complete the proof.
\end{proof}

\begin{corollary}
	Following Theorem \ref{thm1}, we can prove that Algorithm \ref{alg_server} is guaranteed  to converge to
	critical points for the non-convex problem $ \lim\limits_{T\rightarrow \infty}\min\limits_{t \in \{0,...,T-1\}}    \|  \nabla f(w_t) \|_2^2 = 0$, as long as the decreasing $\gamma_t$ satisfies:
	\begin{eqnarray}
		\lim_{T\rightarrow \infty}	\sum\limits_{t=0}^{T-1}   \gamma_t = \infty \hspace{0.7cm} and  \hspace{0.7cm}\lim_{T\rightarrow \infty}	\sum\limits_{t=0}^{T-1}   \gamma_t^2 < \infty.
		\label{iq_30001}
	\end{eqnarray}
	\label{thm_11}
\end{corollary}
The results above follow from the seminal work in \cite{robbins1951stochastic} and the two requirements in (\ref{iq_30001}) can be easily satisfied if we let $\gamma_t= \frac{1}{t+1}$. We can also obtain the convergence rate of FedAvg if we let $\gamma_t$ be a constant.

\begin{corollary}
	Following Theorem \ref{thm1}, we suppose stepsize $\gamma_t=\gamma$ for all $t \in \{0,...,T-1 \}$ and $\eta = 1$. If   $\gamma \leq \min \left\{ \sqrt{ \frac{4K(f(w_0)-f_{\inf}) }{3MTHL\sigma^2} }, \frac{1}{4LH}  \right\} $, it is  guaranteed that Algorithm \ref{alg_server} converges as follows:
	\begin{eqnarray}
		\min\limits_{t \in \{0,...,T-1\}}    \|  \nabla f(w_t) \|_2^2  &\leq &  \frac{8LK (f(w_0) - f_{\inf}) }{MT}   \nonumber \\
		&&+   \sqrt{ \frac{12KL\sigma^2(f(w_0)-f_{\inf}) }{MTH} }. \nonumber
	\end{eqnarray}
	\label{thm_12}
\end{corollary}

We have proven that FedAvg is guaranteed to converge to critical points for non-convex problems at $O\left(\sqrt{\frac{1}{T}}\right)$.

This is the first work which confirms the convergence of FedAvg for non-convex problems. It is worthwhile to highlight the generalities of our analysis compared to \cite{sahu2018convergence} as follows:
$i$) no data distribution assumption, so that it is satisfied for clients with any data distributions;
$ii$) no constraints on local subproblem, so subproblems can also be non-convex.

\section{New Federated Momentum Algorithm}
\label{sec::fedmom}
By understanding model averaging as a gradient-based method with biased gradients, we propose a novel accelerated federated momentum algorithm on the server end. We also prove that the proposed method is guaranteed to converge to critical points for non-convex problems.

\subsection{Algorithm Description}
The origin of momentum methods dates back to the 1960's \cite{polyak1964some}. Since then, it has achieved the optimal convergence rate for strongly convex smooth optimization \cite{nesterov1983method,allen2014linear}. Although admitting a similar convergence rate as SGD for non-convex problems, momentum methods exhibit impressive performance in training deep neural networks \cite{yang2016unified}. In Section \ref{sec::fedavg}, we reformulate  model averaging in Algorithm \ref{alg_server} as an update of gradient-based methods.  To accelerate the training of federated learning, we propose Federated Momentum algorithm (FedMom) by using Nesterov's accelerated gradient on the server.

We describe the procedures of FedMom in Algorithm \ref{alg_mom}. FedMom is similar to FedAvg in selecting clients and receiving updated models at each iteration. However, instead of computing the average of collected models, the server stores a momentum variable and updates the server model following steps 8-9 in Algorithm \ref{alg_mom}:
\begin{eqnarray}
	v_{t+1} &= &w_t - \eta  \sum\limits_{k=1}^K  \frac{n_k}{n }\left( w_t - w_{t+1}^k  \right), \nonumber\\
	w_{t+1} &=& v_{t+1} + \beta (v_{t+1} - v_t). \nonumber
\end{eqnarray}
In other words, FedMom performs a simple step like SGD from $w_t$ to $v_{t+1}$ at first. After that, the model moves a little bit further in the direction of the previous point $v_t$. Parameter $\beta$ is selected from $[0,1)$. In the experiment, we set $\beta=0.9$ all the time. In the following context, we provide convergence guarantees of FedMom for non-convex problems.

\begin{algorithm}[t]
	\renewcommand{\algorithmicrequire}{\textbf{Phase 2:}}
	\renewcommand{\algorithmicensure}{\textbf{Initialize:}}
	\caption{Federated Momentum (FedMom) \hfill [Server]}
	\begin{algorithmic}[1]
		\ENSURE  $v_0 = w_0$, $\eta \in \left[1, \frac{K}{M} \right]$; 
		\FOR{$t=0,1,\cdots,T-1$}
		\STATE  $S_t =$  random set of M clients;
		\FOR{each client $k \in S_t$ \textbf{in parallel}}
		\STATE Send $w_t$ to client $k$;
		\STATE  Receive $w^k_{t+1} $ from client $k$ through Algorithm \ref{alg_worker};
		\ENDFOR	
		\STATE $w_{t+1}^k = w_t$ for any $k \notin S_t$	;
		\STATE Update momentum vector $v$:
			\begin{eqnarray}
v_{t+1} &= &w_t - \eta  \sum\limits_{k=1}^K  \frac{n_k}{n }\left( w_t - w_{t+1}^k  \right) ; \nonumber
			\end{eqnarray}
		\STATE Update $w_{t+1} = v_{t+1} + \beta (v_{t+1} - v_t)  $.
		\ENDFOR	
	\end{algorithmic}
	\label{alg_mom}		
\end{algorithm}

\subsection{Convergence Analysis}
Under Assumptions \ref{ass_bd} and \ref{ass_lips}, we show that FedMom is guaranteed to converge to critical points for non-convex problems.

\begin{theorem}
	\label{thm2}
	Assume that Assumptions  \ref{ass_bd} and \ref{ass_lips} hold,  $H_t=H$ for all $t \in \{0,...,T-1 \}$. In addition, we assume that loss $f$ has a lower bound $f_{\inf}$ and let $C =  \frac{ML\eta }{4K(1-\beta)} \sigma^2 + \frac{ML\eta^2}{2K(1-\beta)^2} \sigma^2  + \frac{\beta^4 M^2L\eta^3 }{2K^2(1-\beta)^5} \sigma^2   $.  If we set $\gamma \leq \min  \left\{  \sqrt{ \frac{ f(w_0) - f_{\inf} }{THC}}, \frac{1-\beta}{4\eta HL}, \frac{(1-\beta)^2}{\eta \beta^2HL} \sqrt{\frac{K}{8M}}  \right\} $, the output of Algorithm \ref{alg_mom} satisfies that:
	\begin{eqnarray}
		\min\limits_{t \in \{0,...,T-1\}}  \mathbb{E}_{\xi, k}  \left\| \nabla f(w_t) \right\|_2^2 \leq  \frac{ 16K\eta L ( f(w_0) - f_{\inf} )  }{TM}  \nonumber \\
		+ \frac{4\eta L\beta^2  ( f(w_0) - f_{\inf} ) \sqrt{8K}   }{ (1-\beta) \sqrt{M} }   +  \frac{8K(1-\beta) }{M } \sqrt{\frac{ (f(w_0) - f_{\inf})C}{TH}}. \nonumber
	\end{eqnarray}
\end{theorem}

\begin{proof}
	The procedures of FedMom is as follows:
	\begin{eqnarray}
	\left\{\begin{matrix}
	v_{t+1} &=& w_{t} - \eta \gamma  \sum\limits_{k \in S_t} \frac{n_k}{n} \sum\limits_{h=0}^{H-1}  \nabla f_k (w_{t, h}^k, \xi)   \\
	w_{t+1} & = & v_{t+1} + \beta( {v}_{t+1} - {v}_{t} )
	\end{matrix}\right. .
	\label{newnag}
	\end{eqnarray}
	Defining $w_{-1} = w_0$, $g_t = \eta \gamma  \sum\limits_{k \in S_t} \frac{n_k}{n} \sum\limits_{h=0}^{H-1}  \nabla f_k (w_{t, h}^k, \xi)  $ and $g_{-1} = 0$,  according to (\ref{newnag}), we have:
	\begin{eqnarray}
	w_{t+1} & = &   w_{t} - g_t  + \beta \left(w_t 	- g_t - w_{t-1} + g_{t-1}  \right). \nonumber
	\end{eqnarray}
	We also define $p_t$ as follows:
	\begin{eqnarray}
	p_t &=& \frac{\beta }{1-\beta} \left(w_t - w_{t-1} + g_{t-1} \right),\nonumber
	\end{eqnarray}
	where $p_t=0$. It also holds that:
	\begin{eqnarray}
	p_{t+1} = \beta p_t - \frac{\beta^2}{1-\beta} g_t.
	\label{pt}
	\end{eqnarray}
	Following \cite{yang2016unified}, we can prove that:
	\begin{eqnarray}
	w_{t+1} + p_{t+1} & = & \frac{1}{1-\beta} w_{t+1} - \frac{\beta}{1-\beta} w_{t} + \frac{\beta}{1-\beta} g_t \nonumber \\
	& = & \frac{1}{1-\beta} w_t - \frac{1}{1-\beta} g_t- \frac{\beta}{1-\beta} w_{t-1} + \frac{\beta}{1-\beta} g_{t-1} \nonumber \\
	&=&w_{t} + p_{t} - \frac{1}{1-\beta} g_t.
	\label{iq_001}
	\end{eqnarray}
	Let $z_t = w_t + p_t$, according to Assumption \ref{ass_lips} and taking expectation over $\xi$ and $k$, we know that:
	
	{\small
		\begin{eqnarray}
		&& \mathbb{E}_{\xi, k} [	f(z_{t+1})] \nonumber \\
		&\leq & f(z_t) -  \frac{M \eta \gamma}{K(1-\beta)} \sum\limits_{h=0}^{H-1}  \sum\limits_{k=1}^K  \frac{n_k}{n}   \mathbb{E}_{\xi, k} \left<\nabla f_k (z_t),  \nabla f_k (w_{t, h}^k)    \right> \nonumber \\
		&& + \frac{ML \eta^2\gamma^2}{2K(1-\beta)^2}  \sum\limits_{k=1}^K \frac{n_k}{n}  \mathbb{E}_{\xi, k}  \left\|  \sum\limits_{h=0}^{H-1}   \nabla f_k (w_{t, h}^k, \xi)   \right\|_2^2 \nonumber \\
		&  = &   f(z_t) -  \frac{M \eta \gamma}{K(1-\beta)} \sum\limits_{h=0}^{H-1}  \sum\limits_{k=1}^K  \frac{n_k}{n}  \mathbb{E}_{\xi, k} \left<\nabla f_k(w_t),  \nabla f_k (w_{t, h}^k)    \right>  \nonumber \\
		&& -  \frac{M \eta \gamma}{K(1-\beta)} \sum\limits_{h=0}^{H-1}  \sum\limits_{k=1}^K  \frac{n_k}{n}   \mathbb{E}_{\xi, k}\left<\nabla f_k (z_t) - \nabla f_k(w_t),  \nabla f_k (w_{t, h}^k)    \right> \nonumber \\
		&& + \frac{ML\eta^2\gamma^2}{2K(1-\beta)^2}  \sum\limits_{k=1}^K \frac{n_k}{n}  \mathbb{E}_{\xi, k}  \left\|  \sum\limits_{h=0}^{H-1}   \nabla f_k (w_{t, h}^k, \xi)   \right\|_2^2 \nonumber \\
		&  \leq  &   f(z_t) +  \frac{M \eta\gamma}{K(1-\beta)} \sum\limits_{h=0}^{H-1}  \sum\limits_{k=1}^K  \frac{n_k}{n} \biggl(  \underbrace{  \mathbb{E}_{\xi, k}\left\| \nabla f_k (z_t) - \nabla f_k(w_t)\right\|_2^2}_{Q_3}  \nonumber \\
		&&+  \frac{1}{4}  \left\|  \nabla f_k(w_t)\right\|_2^2 \biggr)   -  \frac{M \eta \gamma}{2K(1-\beta)} \sum\limits_{h=0}^{H-1}  \sum\limits_{k=1}^K  \frac{n_k}{n} \biggl(  \|\nabla f_k(w_t)\|_2^2 \nonumber \\
		&& +  \mathbb{E}_{\xi, k}\|\nabla f_k (w_{t, h}^k)\|_2^2   - \underbrace{ \mathbb{E}_{\xi, k} \|\nabla f_k(w_t) -  \nabla f_k (w_{t, h}^k)  \|_2^2 }_{Q_4}  \biggr)  \nonumber \\
		&& + \frac{ML\eta^2\gamma^2}{2K(1-\beta)^2}  \sum\limits_{k=1}^K \frac{n_k}{n}   \underbrace{  \mathbb{E}_{\xi, k}   \left\| \sum\limits_{h=0}^{H-1}  \nabla f_k (w_{t, h}^k, \xi)   \right\|_2^2 }_{Q_5},
		\label{iq_4001}
		\end{eqnarray}}
	where the inequalities follow from Cauchy's inequality and techniques in the proof of inequalities (\ref{iq_20001}) and (\ref{iq_10001}).  We readily get the upper bound of $Q_5$ from inequality (\ref{iq_601}):
	\begin{eqnarray}
	Q_5 	&  \leq   & H  \sigma^2  + H \sum\limits_{h=0}^{H-1} \mathbb{E}_{\xi, k} \left\| \nabla f_k (w_{t, h}^k)\right\|_2^2. \nonumber
	\end{eqnarray}
	The upper bound of $Q_3$ is as follows:
	\begin{eqnarray}
	Q_3 	&\leq & L^2 \mathbb{E}_{\xi, k}  \left\|  z_t-  w_{t} \right\|_2^2 \nonumber \\
	&= & L^2 \mathbb{E}_{\xi, k}  \left\|  p_t \right\|_2^2, \nonumber
	\end{eqnarray}
	where the first inequality follows from Assumption \ref{ass_lips} and the equality follows from $w_t= w_{t, 0}^k$. Because of (\ref{pt}) and $p_0=0$, it is satisfied that:
	\begin{eqnarray}
	p_t & = &  \beta p_{t-1} - \frac{\beta^2}{1-\beta} g_{t-1} \nonumber \\
	&= & - \frac{\beta^2}{1-\beta}  \sum\limits_{j=0}^{t-1} \beta^{t-1-j} g_j \nonumber \\
	&= & - \frac{\beta^2}{1-\beta}  \sum\limits_{j=0}^{t-1} \beta^{j} g_{t-1-j}. \nonumber
	\end{eqnarray}
	Let $\Lambda_{t} = \sum\limits_{j=0}^{t-1} \beta^j  $, we have:
	\begin{eqnarray}
	\mathbb{E}_{\xi, k} 	\| p_t \|_2^2 & = &  \frac{\beta^4 \Lambda_{t}^2}{(1-\beta)^2}  \mathbb{E}_{\xi, k}  \left\| \sum\limits_{j=0}^{t-1} \frac{\beta^{j}}{\Lambda_{t} }  g_{t-1-j} \right\|_2^2 \nonumber \\
	&\leq &  \frac{\beta^4 \Lambda_{t}}{(1-\beta)^2} \sum\limits_{j=0}^{t-1} {\beta^{j}} \mathbb{E}_{\xi, k}  \left\|   g_{t-1-j} \right\|_2^2,\nonumber
	\end{eqnarray}
	where the inequality is from Jensen's inequality. We can also obtain the upper bound of $\mathbb{E}_{\xi, k} \|g_t\|_2^2$ as follows:
	\begin{eqnarray}
	\mathbb{E}_{\xi, k} \|g_t\|_2^2  =     \mathbb{E}_{\xi, k} \left\|\eta\gamma  \sum\limits_{k \in S_t} \frac{n_k}{n} \sum\limits_{h=0}^{H-1}  \nabla f_k (w_{t, h}^k, \xi) \right\|_2^2 \nonumber \\
	\leq \frac{\eta^2 \gamma^2}{K}    \sum\limits_{k=1}^K  \sum\limits_{k \in S_t} \frac{n_k}{n} \mathbb{E}_{\xi, k}  \left\|   \sum\limits_{h=0}^{H-1}  \nabla f_k (w_{t, h}^k, \xi) \right\|_2^2 \nonumber \\
	\leq \frac{ M\eta^2\gamma^2H}{K}  \sigma^2     +  \frac{ M\eta^2\gamma^2H}{K}    \sum\limits_{k=1}^K  \frac{n_k}{n}  \sum\limits_{h=0}^{H-1} \mathbb{E}_{\xi, k}   \left\|    \nabla f_k (w_{t, h}^k) \right\|_2^2, \nonumber
	\end{eqnarray}
	where the first inequality follows from Jensen's inequality,  the second inequality follows from  $Q_5$ and $  \mathbb{E}_k \sum\limits_{k \in S_t} \frac{n_k}{n} \mathbb{E}_{\xi, k}  \left\|   \sum\limits_{h=0}^{H-1}  \nabla f_k (w_{t, h}^k, \xi) \right\|_2^2 =  \frac{M}{K}\sum\limits_{k=1}^K  \frac{n_k}{n}  \mathbb{E}_{\xi, k}  \left\|   \sum\limits_{h=0}^{H-1}  \nabla f_k (w_{t, h}^k, \xi) \right\|_2^2 $. 
	Therefore, we obtain that:
	\begin{eqnarray}
	\mathbb{E}_{\xi, k} \| p_t \|_2^2   \leq    \frac{\beta^4 \Lambda_{t}}{(1-\beta)^2} \cdot \frac{ M\eta^2\gamma^2H}{K}  \sigma^2   \sum\limits_{j=0}^{t-1} {\beta^{j}}  \nonumber \\
	+   \frac{\beta^4 \Lambda_{t}}{(1-\beta)^2} \cdot  \frac{ M\eta^2\gamma^2H}{K}       \sum\limits_{k=1}^K  \frac{n_k}{n}  \sum\limits_{h=0}^{H-1}  \sum\limits_{j=0}^{t-1} {\beta^{j}} 	\mathbb{E}_{\xi, k}  \left\|    \nabla f_k (w_{t-1-j, h}^k) \right\|_2^2.
	\label{iq_3001}
	\end{eqnarray}
	Because $\Lambda_{t} = \sum\limits_{j=0}^{t-1} \beta^j = \frac{1-\beta^t}{1-\beta} \leq \frac{1}{1-\beta} $ and summing inequality (\ref{iq_3001}) from $t=0$ to $T-1$, we have:
	\begin{eqnarray}
	&&\sum\limits_{t=0}^{T-1}   	\mathbb{E}_{\xi, k} \| p_t \|_2^2 \nonumber \\
	& \leq &  \frac{\beta^4 }{(1-\beta)^4} \cdot \frac{ M\eta^2\gamma^2TH}{K}  \sigma^2    \nonumber \\
	&& +    \frac{\beta^4 }{(1-\beta)^3} \cdot  \frac{ M\eta^2\gamma^2H}{K}    \sum\limits_{k=1}^K  \frac{n_k}{n}  \sum\limits_{h=0}^{H-1}   \sum\limits_{t=0}^{T-1}  \sum\limits_{j=0}^{t-1} {\beta^{j}} \left\|    \nabla f_k (w_{t-1-j, h}^k) \right\|_2^2 \nonumber\\
	& = &  \frac{\beta^4 }{(1-\beta)^4} \cdot \frac{ M\eta^2\gamma^2TH}{K}  \sigma^2  \nonumber \\
	&& +    \frac{\beta^4 }{(1-\beta)^3} \cdot  \frac{ M\eta^2\gamma^2H}{K}   \sum\limits_{k=1}^K  \frac{n_k}{n}  \sum\limits_{h=0}^{H-1}   \sum\limits_{t=0}^{T-1}  \left\|    \nabla f_k (w_{t, h}^k) \right\|_2^2  \sum\limits_{j=t}^{T-1} {\beta^{T-1-j}}  \nonumber \\
	& \leq & \frac{\beta^4 }{(1-\beta)^4} \cdot \frac{ M\eta^2\gamma^2TH}{K}  \sigma^2  \nonumber \\
	&& +    \frac{\beta^4 }{(1-\beta)^4} \cdot  \frac{ M\eta^2\gamma^2H}{K}     \sum\limits_{k=1}^K  \frac{n_k}{n}  \sum\limits_{h=0}^{H-1}   \sum\limits_{t=0}^{T-1}  \left\|    \nabla f_k (w_{t, h}^k) \right\|_2^2. \nonumber
	\end{eqnarray}
	To obtain the upper bound of $Q_4$, we have:
	\begin{eqnarray}
	Q_4 &\leq & L^2 	\mathbb{E}_{\xi,k}  \left\| \sum\limits_{j=0}^{h-1} \left( w^k_{t,j+1} - w_{t,j}^k \right) \right\|_2^2  \nonumber \\
	& = &  L^2 \gamma^2  \mathbb{E}_{\xi,k}\left\| \sum\limits_{j=0}^{h-1}\left( \nabla  f_k (w_{t,j}^k, \xi)   - \nabla  f_k (w_{t,j}^k) + \nabla  f_k (w_{t,j}^k) \right) \right\|_2^2 \nonumber \\
	& \leq & h L^2 \gamma^2 \sigma^2  + h L^2 \gamma^2 \sum\limits_{j=0}^{h-1}  \mathbb{E}_{\xi,k}\left\|  \nabla  f_k (w_{t,j}^k)  \right\|_2^2, \nonumber
	\end{eqnarray}
	where the first inequality follows from Assumption \ref{ass_lips}, the second inequality follows from  Assumption \ref{ass_bd} and  $\mathbb{E} \|z_1 + ... + z_n \|_2^2 \leq \mathbb{E} \left[   \|z_1\|_2^2 + ... + \|z_n\|_2^2\right]$	for any random variable $z_1$, ... , $z_n$ with mean $0$.
	Summing up $Q_4$ from $h=0$ to $H-1$, it holds that:
	\begin{eqnarray}
	\sum\limits_{h=0}^{H-1} Q_4 
	&\leq & H^2 L^2 \gamma^2 \sigma^2  + H^2L^2\gamma^2 \sum\limits_{h=0}^{H-1} \mathbb{E}_{\xi,k} \left\| \nabla  f_k (w_{t,h}^k)   \right\|_2^2. \nonumber
	\end{eqnarray}

	\begin{figure*}[!htbp]
		\centering
		\begin{subfigure}[b]{0.45\textwidth}
			\centering
			\includegraphics[width=3.2in]{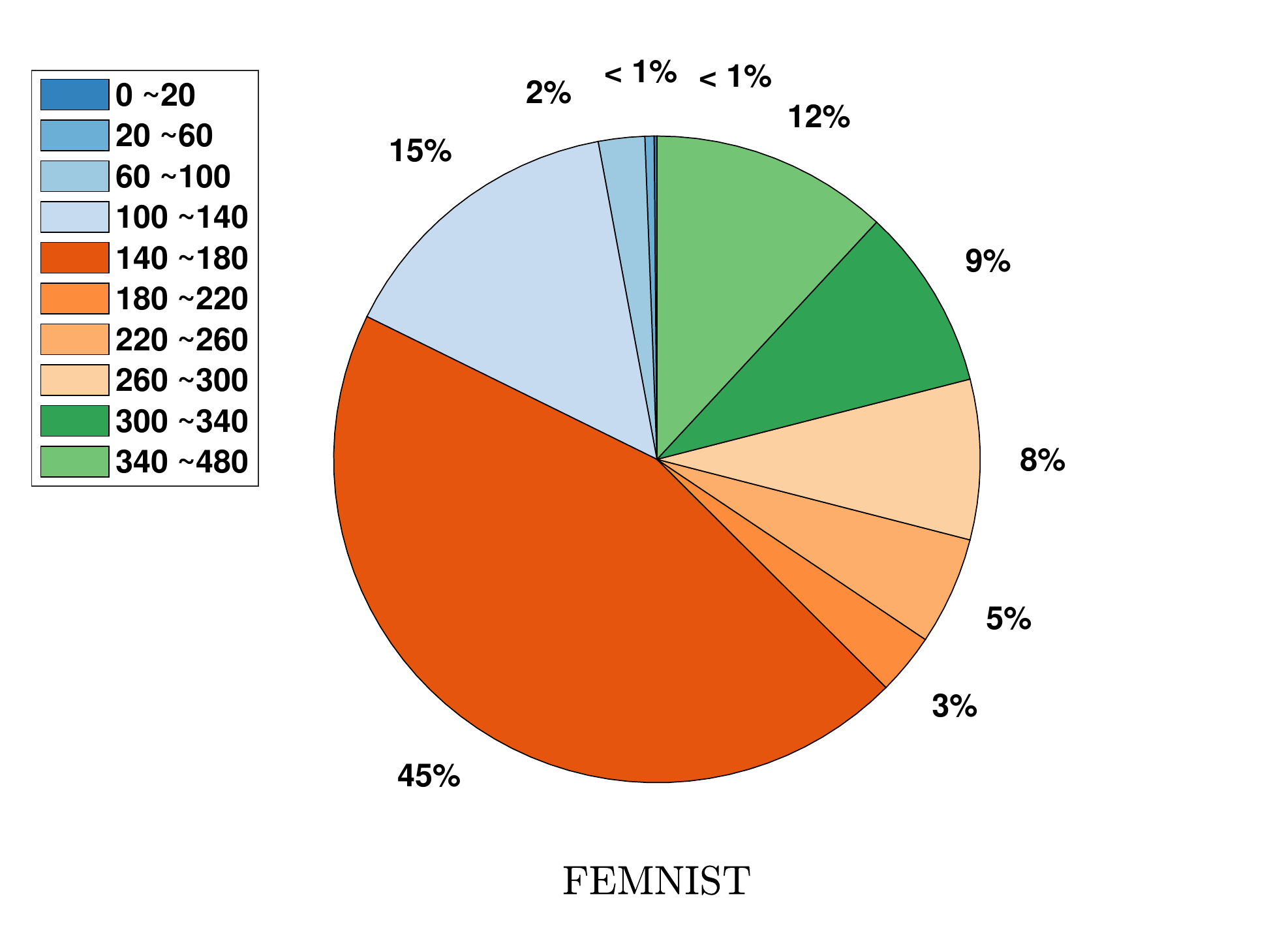}
		\end{subfigure}
		\begin{subfigure}[b]{0.45\textwidth}
			\centering
			\includegraphics[width=3.2in]{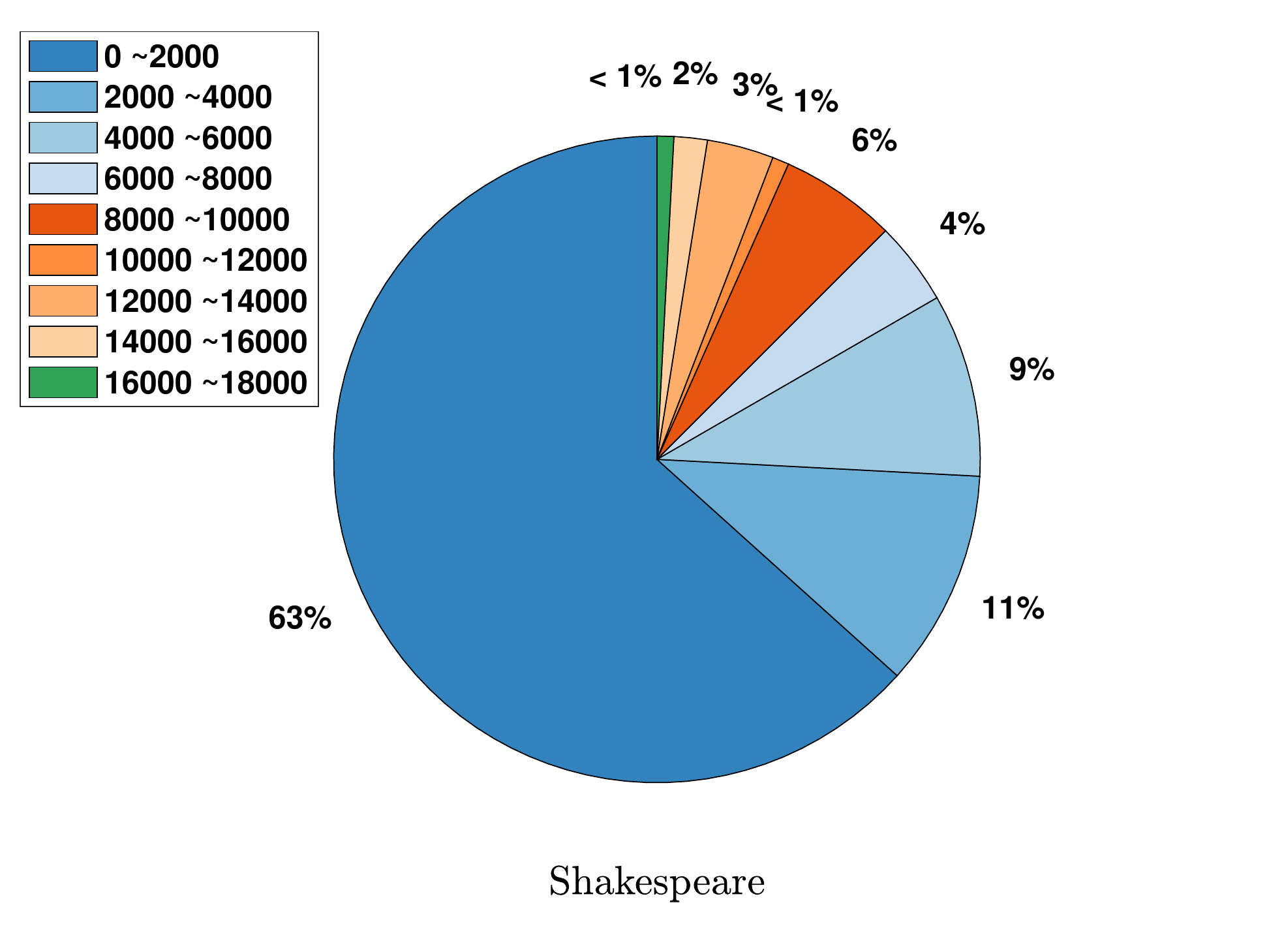}
		\end{subfigure}
		\caption{Visualization of non-IID  and unbalanced data  on all clients.  For FEMNIST dataset, the number of samples on clients is from $0$ to $480$; for Shakespeare dataset, the number of samples on clients is from $0$ to $18000$. }
		\label{fig::pie}
	\end{figure*}
	
	\begin{table*}[pt]
		\center
		\caption{Statistics of FEMNIST and Shakespeare datasets used in our experiment.}
		\setlength{\tabcolsep}{3mm}
		\begin{tabular}{cccccc}
			\hline		
			\multirow{2}{*}{ \textbf{Dataset}} 	& 	\multirow{2}{*}{\textbf{Type}}  & 	\multirow{2}{*}{\textbf{$\#$ samples}}  & 	\multirow{2}{*}{\textbf{$\#$ clients} }  & \multicolumn{2}{c}{\textbf{Statistics Per Client}} \\
			&&& & \textbf{Mean} & \textbf{Standard Deviation} \\
			\hline
			FEMNIST & Image &785,733 &	3,500 &  224.50 & 87.80 \\
			Shakespeare&   Text & 517,106 & 125&  4136.85 & 7226.20\\
			\hline		
		\end{tabular}
		\label{table:data}
	\end{table*}

	Inputting $Q_3$, $Q_4$ and $Q_5$ into inequality (\ref{iq_4001}) and summing it up from $t=0$ to $T-1$, we have:
	\begin{eqnarray}
	&&\mathbb{E}_{\xi, k} [	f(z_{T})] \nonumber \\
	& \leq & f(z_0) - \frac{MH\eta \gamma}{4K(1-\beta)}  \sum\limits_{t=0}^{T-1} \sum\limits_{k=1}^K \frac{n_k}{n}  \mathbb{E}_{\xi, k}  \left\| \nabla f_k(w_t) \right\|_2^2 \nonumber \\
	&& + \frac{MH\eta \gamma}{K(1-\beta)}   \sum\limits_{k=1}^K \frac{n_k}{n} \sum\limits_{t=0}^{T-1}   Q_3 \nonumber \\
	&& + \frac{M\eta \gamma}{2K(1-\beta)}  \sum\limits_{t=0}^{T-1}  \sum\limits_{k=1}^K \frac{n_k}{n}   \sum\limits_{h=0}^{H-1} Q_4 \nonumber \\
	&& - \frac{M\eta\gamma}{2K(1-\beta)}  \sum\limits_{t=0}^{T-1} \sum\limits_{h=0}^{H-1}   \sum\limits_{k=1}^K \frac{n_k}{n}    \mathbb{E}_{\xi, k}\left\| \nabla  f_k (w_{t,h}^k)   \right\|_2^2 \nonumber \\
	&&+ \frac{ML\eta^2\gamma^2}{2K(1-\beta)^2}   \sum\limits_{t=0}^{T-1}   \sum\limits_{k=1}^K \frac{n_k}{n}Q_5 \nonumber \\
	& \leq & f(z_0) - \frac{MH\gamma}{4K(1-\beta)}  \sum\limits_{t=0}^{T-1}\mathbb{E}_{\xi, k}   \left\| \nabla f(w_t) \right\|_2^2 \nonumber \\
	&&- \bigg(  \frac{M\eta\gamma}{2K(1-\beta)} - \frac{ML^2H^2\eta\gamma^3}{2K(1-\beta)} - \frac{MLH\eta^2\gamma^2}{2K(1-\beta)^2}  \nonumber \\
	&& -  \frac{\beta^4M^2L^2H^2\eta^3\gamma^3}{K^2(1-\beta)^5}  \bigg)  \sum\limits_{t=0}^{T-1} \sum\limits_{h=0}^{H-1}   \sum\limits_{k=1}^K \frac{n_k}{n}  \mathbb{E}_{\xi, k}  \left\| \nabla  f_k (w_{t,h}^k)   \right\|_2^2 \nonumber \\
	&& + \frac{TML^2H^2\eta\gamma^3}{2K(1-\beta)} \sigma^2 + \frac{TMLH\eta^2\gamma^2}{2K(1-\beta)^2} \sigma^2  + \frac{\beta^4 TM^2L^2H^2\eta^3 \gamma^3 }{K^2(1-\beta)^5}  \sigma^2. \nonumber
	\end{eqnarray}
	As long as the following inequalities are satisfied:
	\begin{eqnarray}
	L^2H^2\gamma^2 &\leq & \frac{1}{4}, \nonumber \\
	\frac{LH\eta\gamma}{1-\beta} &\leq & \frac{1}{4}, \nonumber\\
	\frac{\beta^4 ML^2H^2\eta^2 \gamma^2}{K(1-\beta)^4} &\leq  & \frac{1}{8},\nonumber
	\end{eqnarray}
	we have:
	\begin{eqnarray}
	\frac{M\eta\gamma}{2K(1-\beta)} - \frac{ML^2H^2\eta\gamma^3}{2K(1-\beta)} - \frac{MLH\eta^2\gamma^2}{2K(1-\beta)^2}   -  \frac{\beta^4M^2L^2H^2\eta^3\gamma^3}{K^2(1-\beta)^5}    >  0.\nonumber
	\end{eqnarray}
	Thus, it follows that:
	\begin{eqnarray}
	\mathbb{E}_{\xi, k} [	f(z_{T})]
	\leq  f(z_0) - \frac{MH\gamma}{4K(1-\beta)}  \sum\limits_{t=0}^{T-1} \mathbb{E}_{\xi, k}  \left\| \nabla f(w_t) \right\|_2^2 \nonumber \\
	+ \frac{TMLH\eta \gamma^2}{4K(1-\beta)} \sigma^2 + \frac{TMLH\eta^2\gamma^2}{2K(1-\beta)^2} \sigma^2  + \frac{\beta^4 TM^2LH\eta^3\gamma^2 }{2K^2(1-\beta)^5}  \sigma^2.
	\label{iq_5001}
	\end{eqnarray}
	Rearranging inequality (\ref{iq_5001}) and letting $C =  \frac{ML\eta }{4K(1-\beta)} \sigma^2 + \frac{ML\eta^2}{2K(1-\beta)^2} \sigma^2  + \frac{\beta^4 M^2L\eta^3 }{2K^2(1-\beta)^5} \sigma^2   $, we obtain that:
	\begin{eqnarray}
	\min\limits_{t \in \{0,...,T-1\}}  \mathbb{E}_{\xi, k}  \left\| \nabla f(w_t) \right\|_2^2 & \leq&  \frac{4K(1-\beta) ( f(w_0) - f_{\inf} )}{THM\gamma } \nonumber \\
	&&+  \frac{4KC(1-\beta)\gamma }{M }.\nonumber
	\end{eqnarray}
	After inputting the upper bound of  $\gamma$  in the above inequality,	we complete the proof.
	
\end{proof}

Above we also prove that FedMom is guaranteed to converge to critical points for non-convex problems at $O\left(\sqrt{\frac{1}{T}}\right)$. Although FedMom shares the similar convergence rate to FedAvg, we will show in the following context that FedMom works better than FedAvg empirically.

\begin{figure*}[pt]
	\centering
	\begin{subfigure}[b]{0.42\textwidth}
		\centering
		\includegraphics[width=3in]{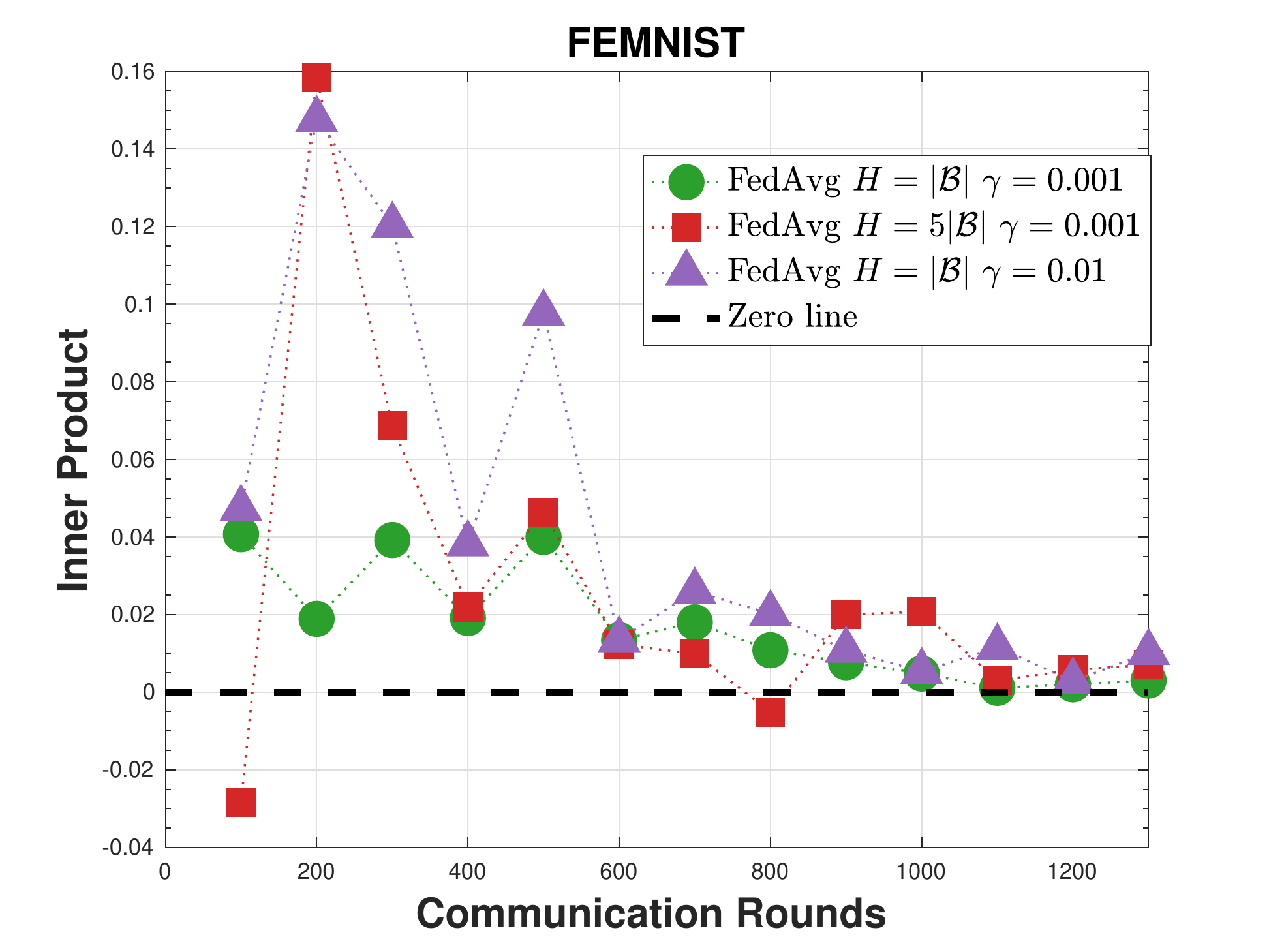}
	\end{subfigure}
	\begin{subfigure}[b]{0.42\textwidth}
		\centering
		\includegraphics[width=3in]{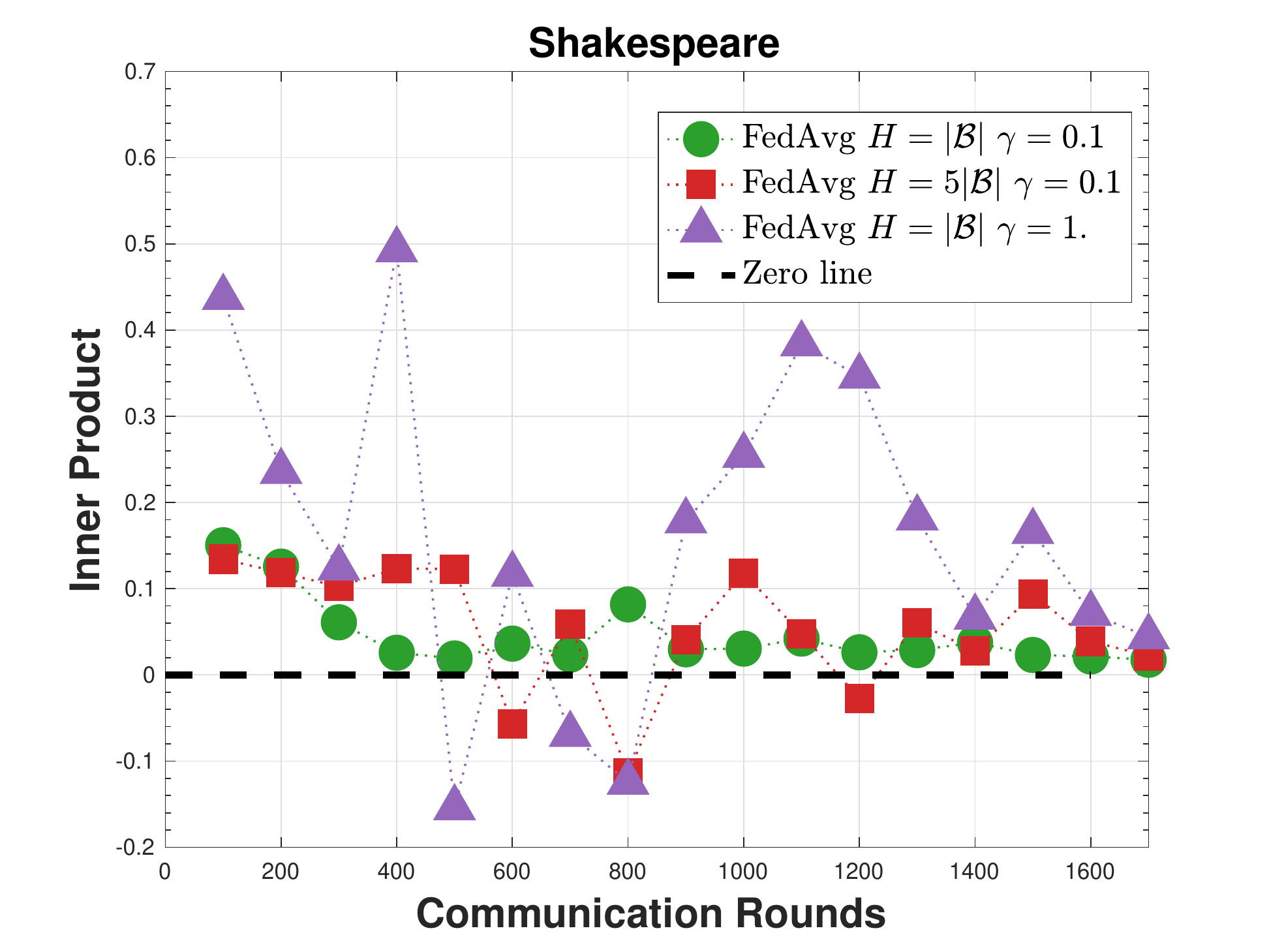}
	\end{subfigure}
	\caption{Variation of  the expectation of Inner product $ \mathbb{E}_{S_t}\left<g_t, w_t - w^*  \right> $ in the course of optimization. We set the model after $2000$ communication rounds as $w^*$.  }
	\label{fig::sigma_femnist}
\end{figure*}

\begin{figure*}[!h]
	\centering
	\begin{subfigure}[b]{0.328\textwidth}
		\centering
		\includegraphics[width=2.5in]{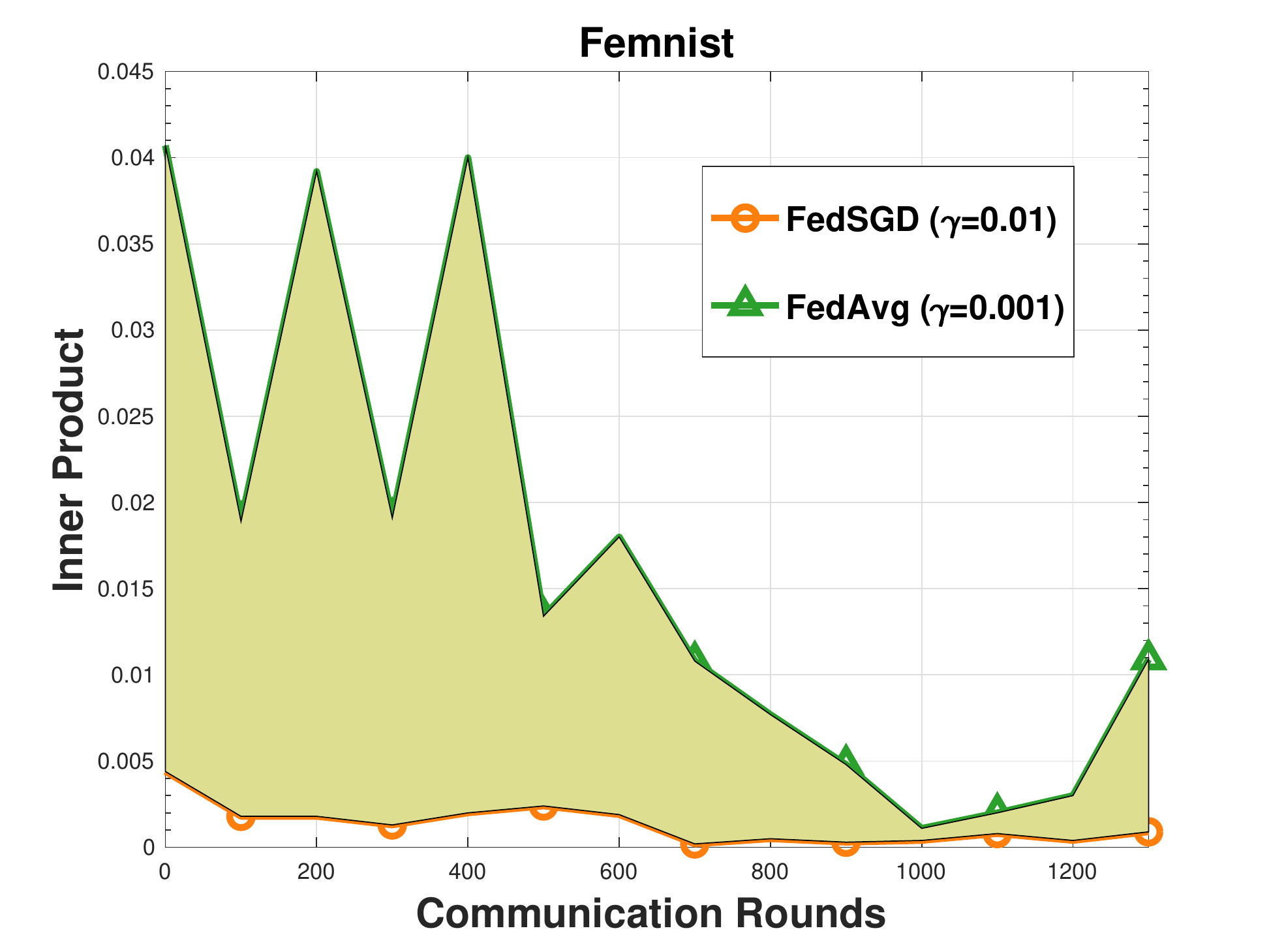}
	\end{subfigure}
	\begin{subfigure}[b]{0.328\textwidth}
		\centering
		\includegraphics[width=2.5in]{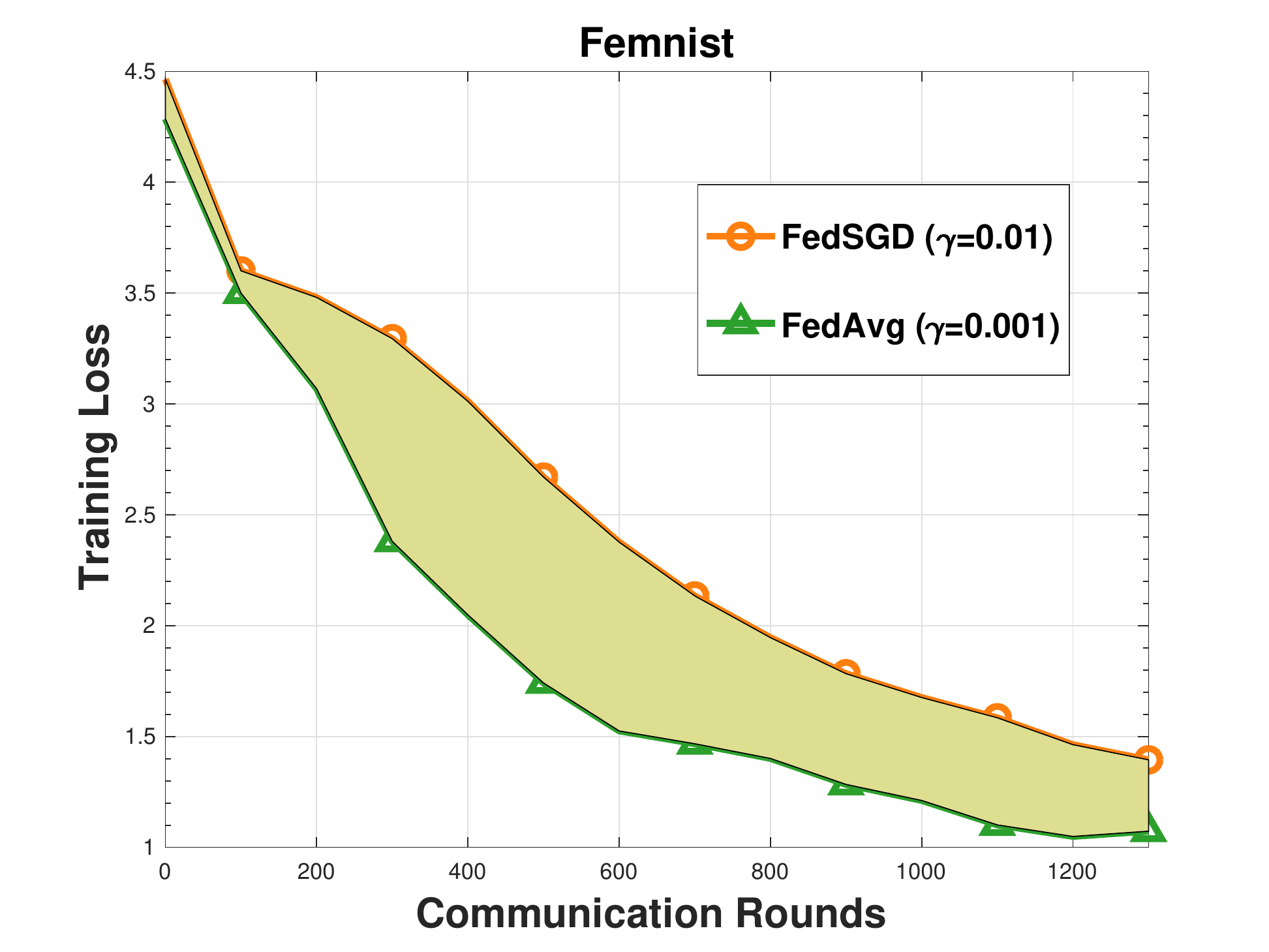}
	\end{subfigure}
	\begin{subfigure}[b]{0.328\textwidth}
		\centering
		\includegraphics[width=2.5in]{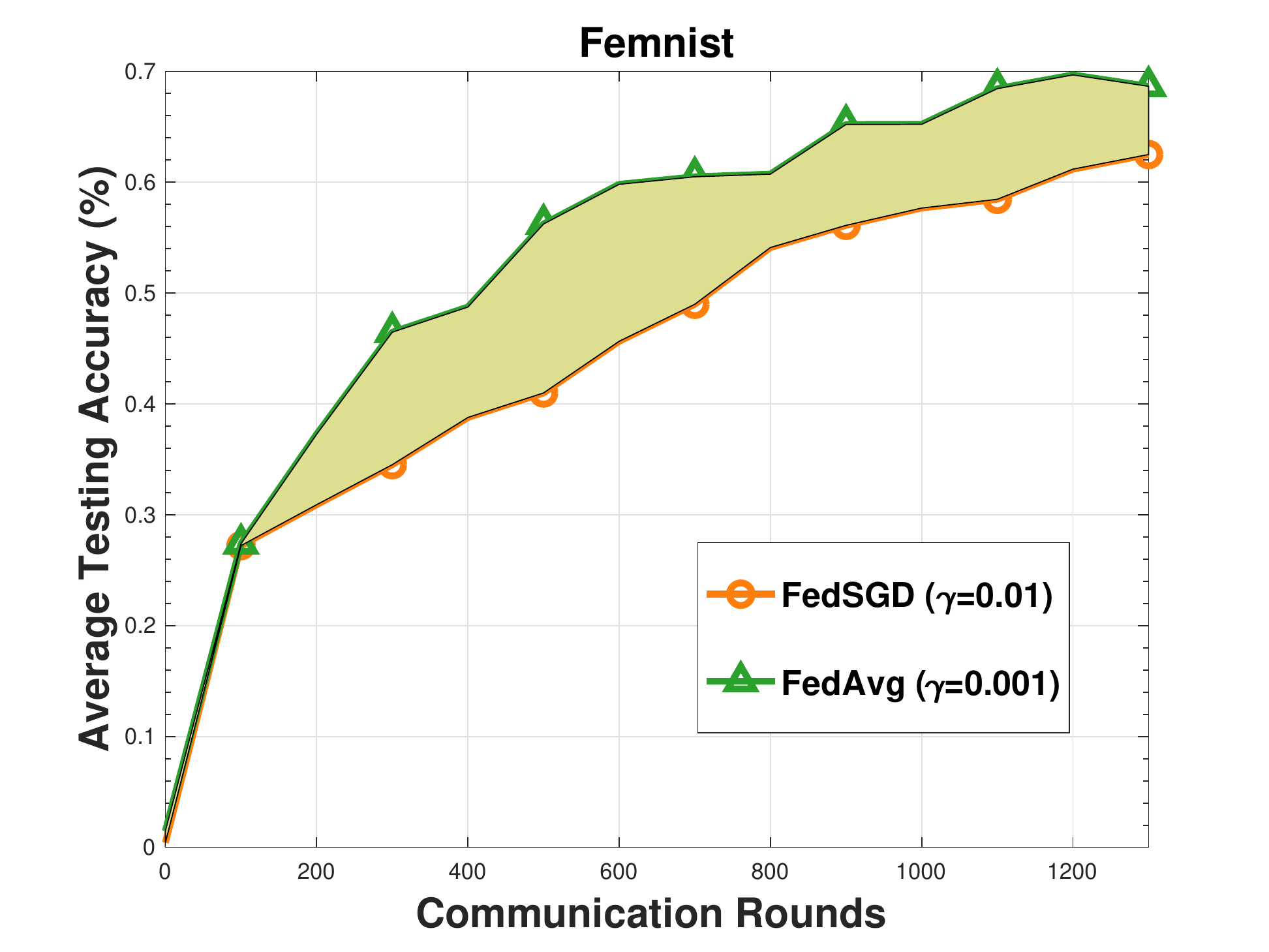}
	\end{subfigure}
	\caption{Verifying why FedAvg converges faster than FedSGD. The shaded region denotes the gap of performance between two methods. ``Inner Product'' represents $ \mathbb{E}_{S_t}\left<g_t, w_t - w^*  \right> $.}
	\label{fig::fedavg_sgd}
\end{figure*}

\begin{figure*}[t]
	\centering
	\begin{subfigure}[b]{0.328\textwidth}
		\centering
		\includegraphics[width=2.5in]{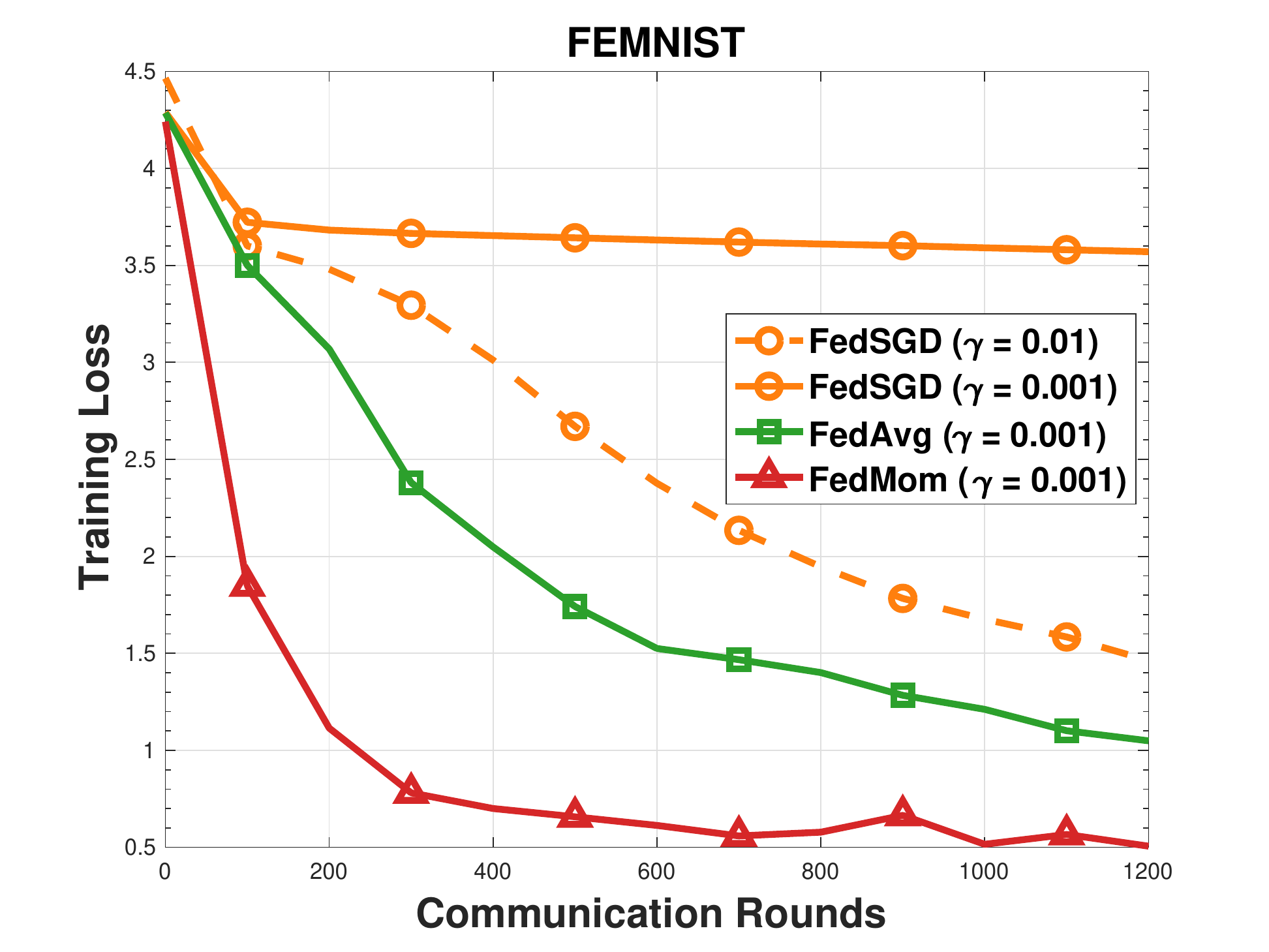}
	\end{subfigure}
	\begin{subfigure}[b]{0.328\textwidth}
		\centering
		\includegraphics[width=2.5in]{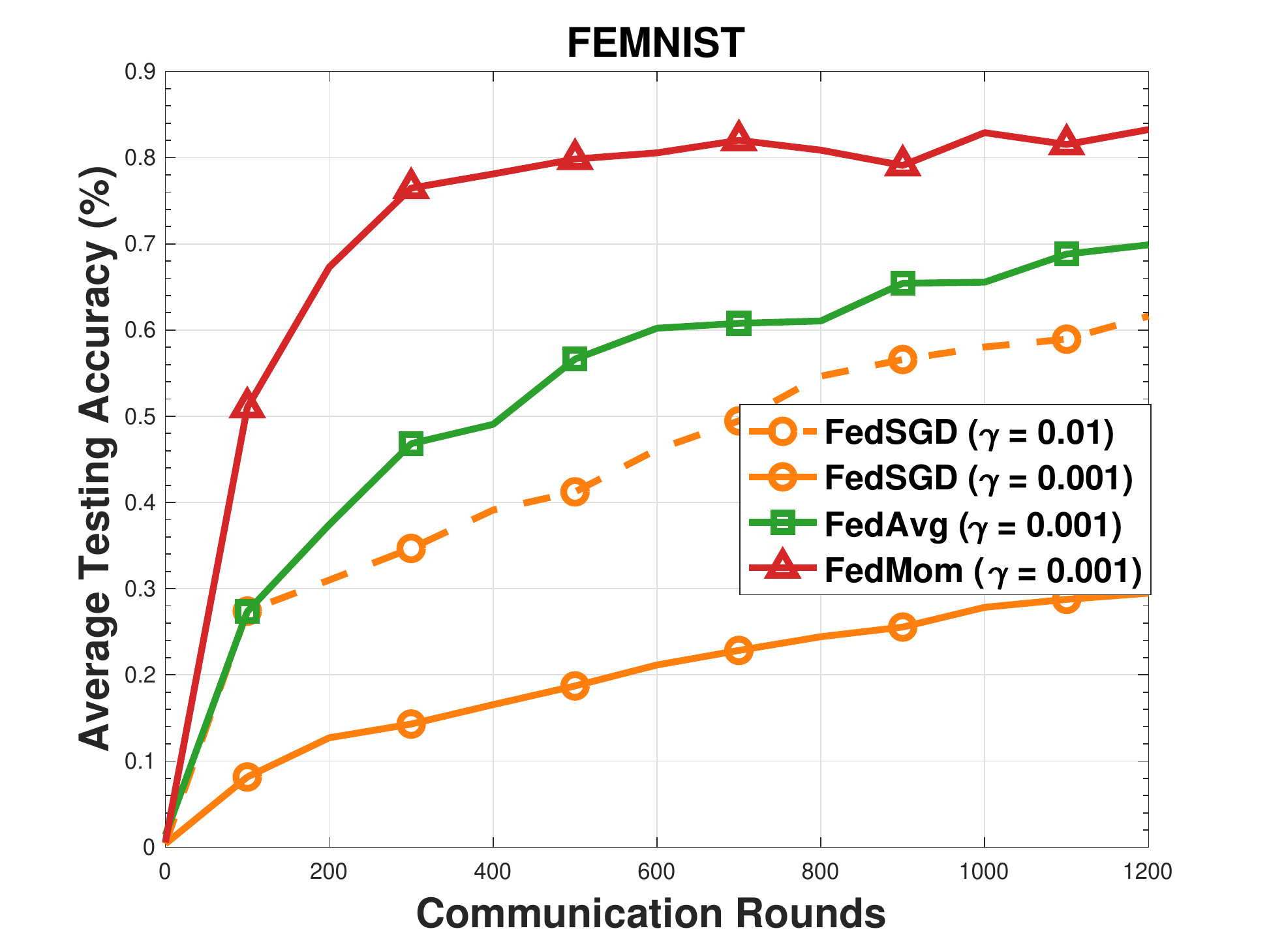}
	\end{subfigure}
	\begin{subfigure}[b]{0.328\textwidth}
		\centering
		\includegraphics[width=2.5in]{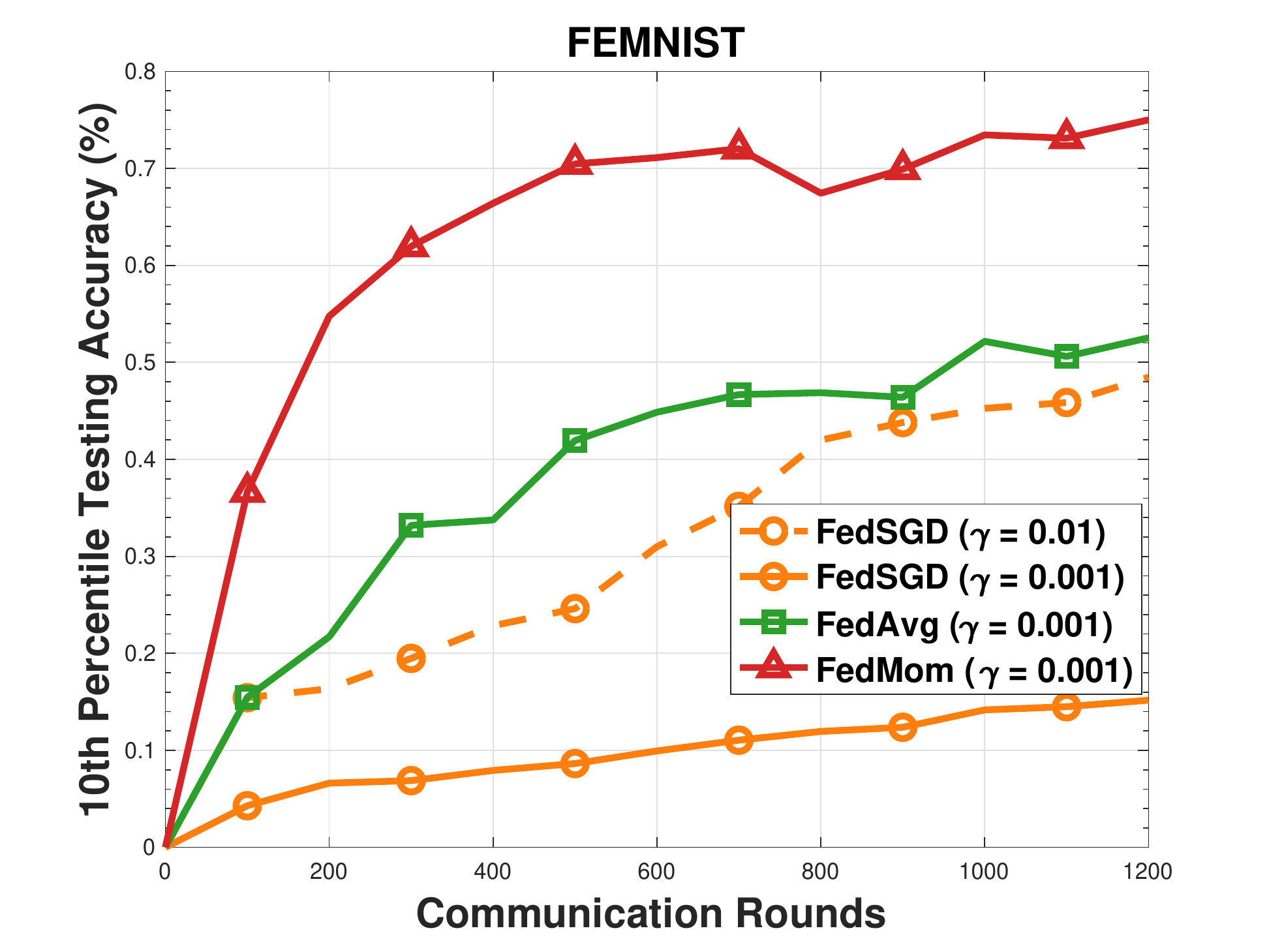}
	\end{subfigure}
	\begin{subfigure}[b]{0.328\textwidth}
		\centering
		\includegraphics[width=2.5in]{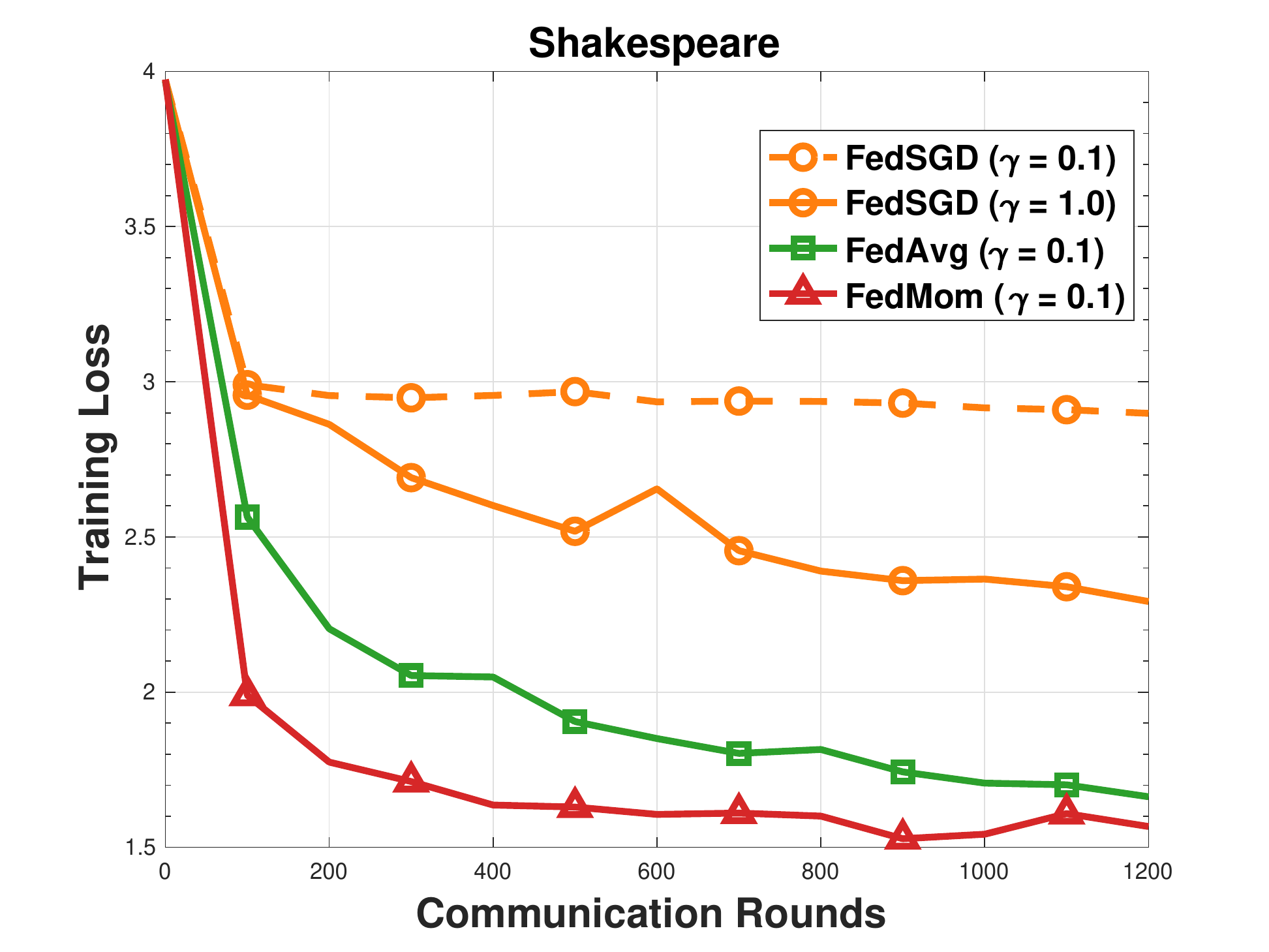}
	\end{subfigure}
	\begin{subfigure}[b]{0.328\textwidth}
		\centering
		\includegraphics[width=2.5in]{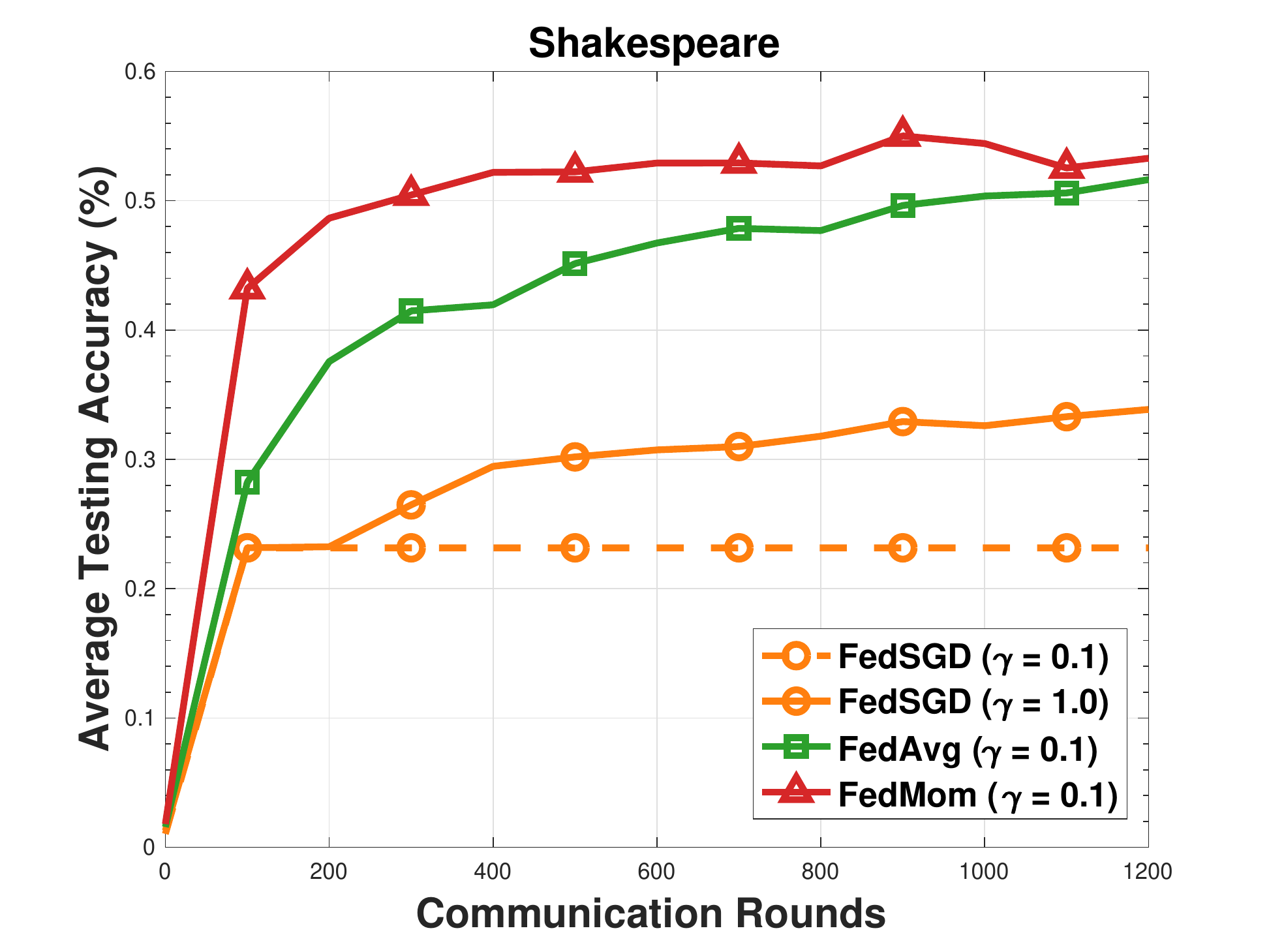}
	\end{subfigure}
	\begin{subfigure}[b]{0.328\textwidth}
		\centering
		\includegraphics[width=2.5in]{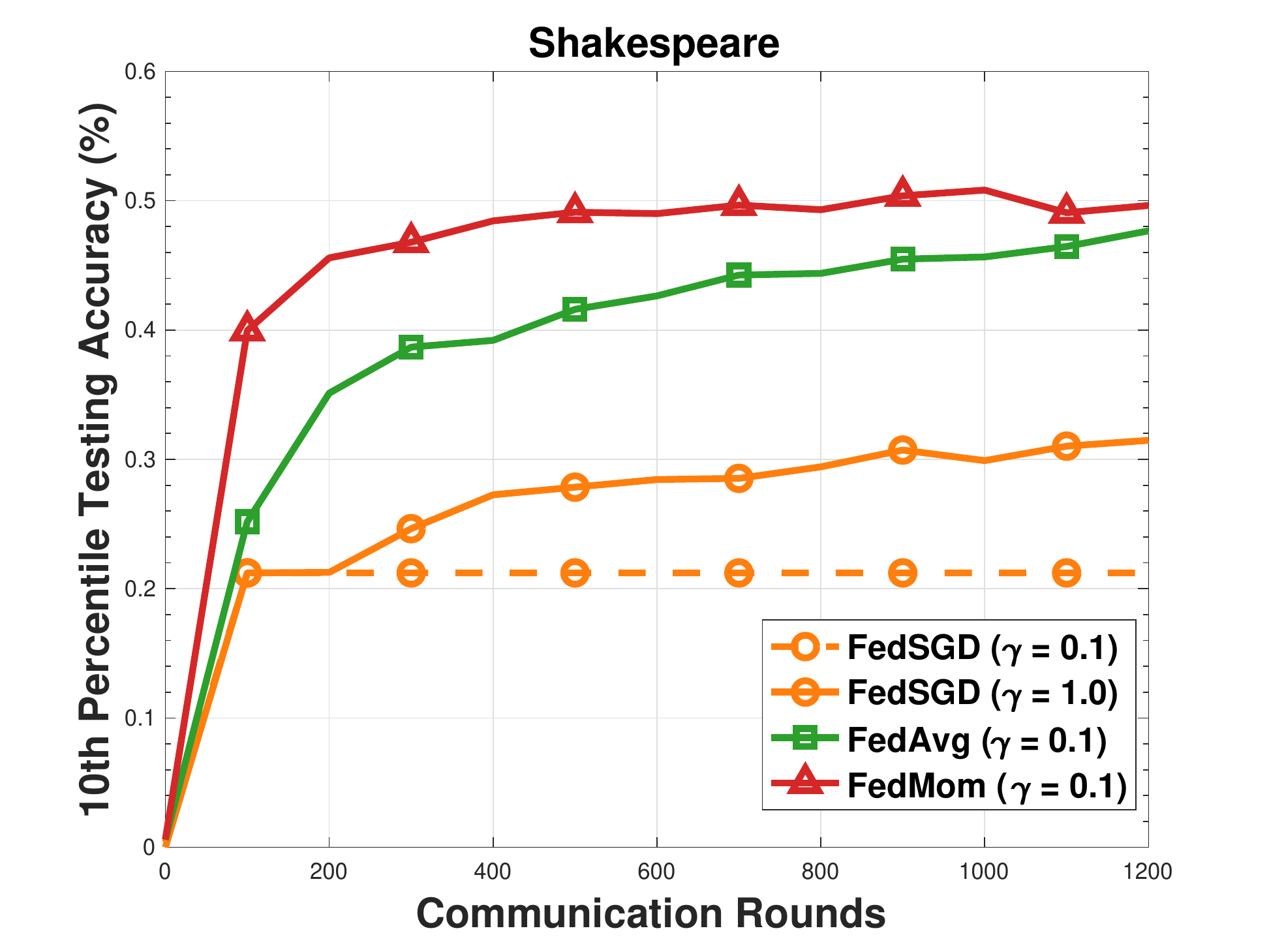}
	\end{subfigure}
	\caption{Performance of compared methods on FEMNIST and Shakespeare dataset. 10th percentile denotes that there are $10\%$ of the data values below it.   }
	\label{fig::cmp_femnist}
\end{figure*}

\section{Experiments}
\label{sec::exp}
We validate our analysis with simulated federated learning experiments on training deep neural networks. There are two targets: ($i$) we verify that the stochastic gradient $g_t$ in (\ref{iq_1_0000}) is a right direction towards target solution although it is biased; ($ii$) we demonstrate that our proposed method converges faster. All experiments are performed on a machine with Intel(R) Xeon(R) CPU E5-2650 v4 @ 2.20GHz and $4$ TITAN Xp GPUs.

\begin{figure*}[h]
	\centering
	\begin{subfigure}[b]{0.38\textwidth}
		\centering
		\includegraphics[width=2.8in]{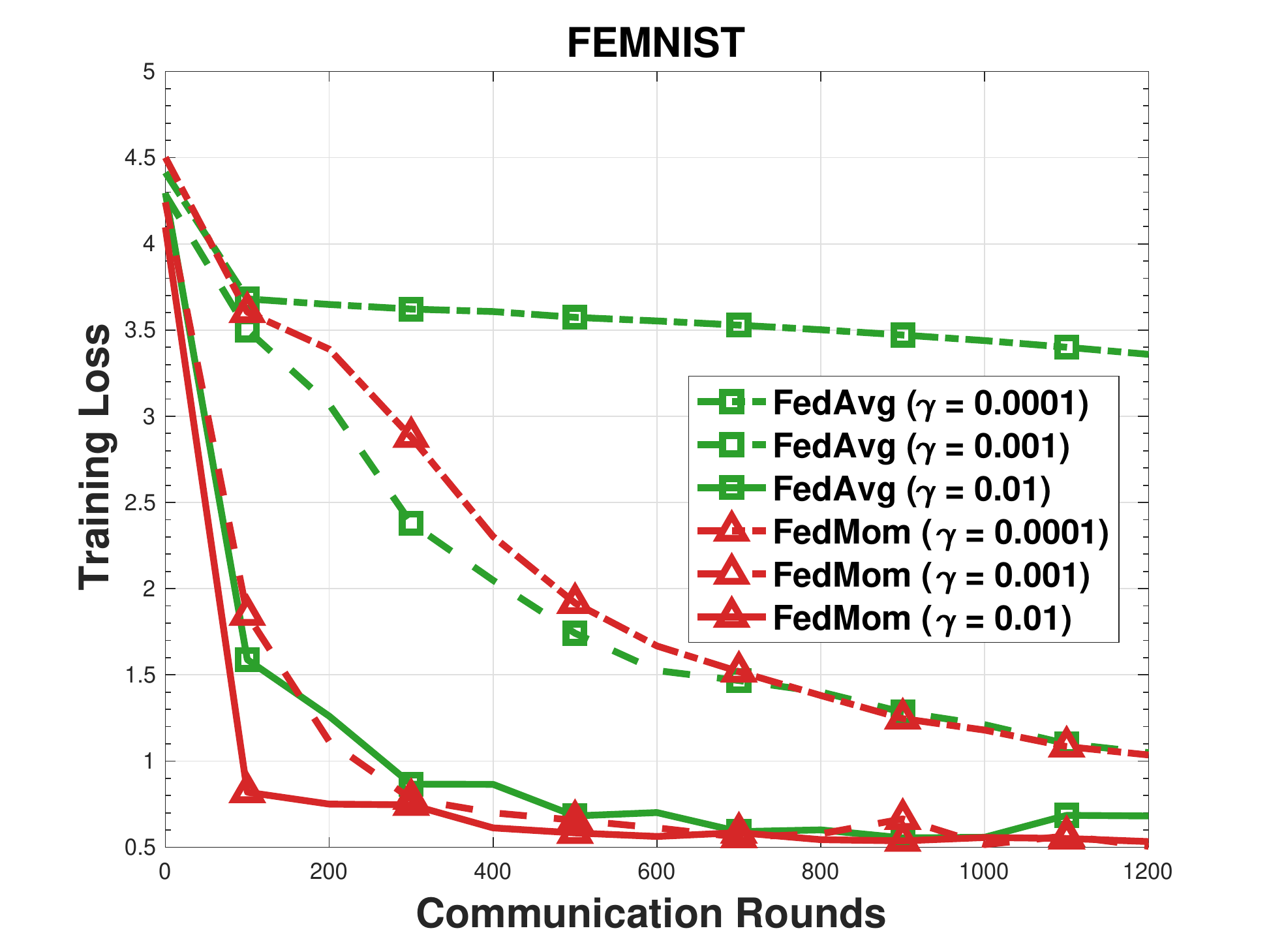}
	\end{subfigure}
	\begin{subfigure}[b]{0.38\textwidth}
		\centering
		\includegraphics[width=2.8in]{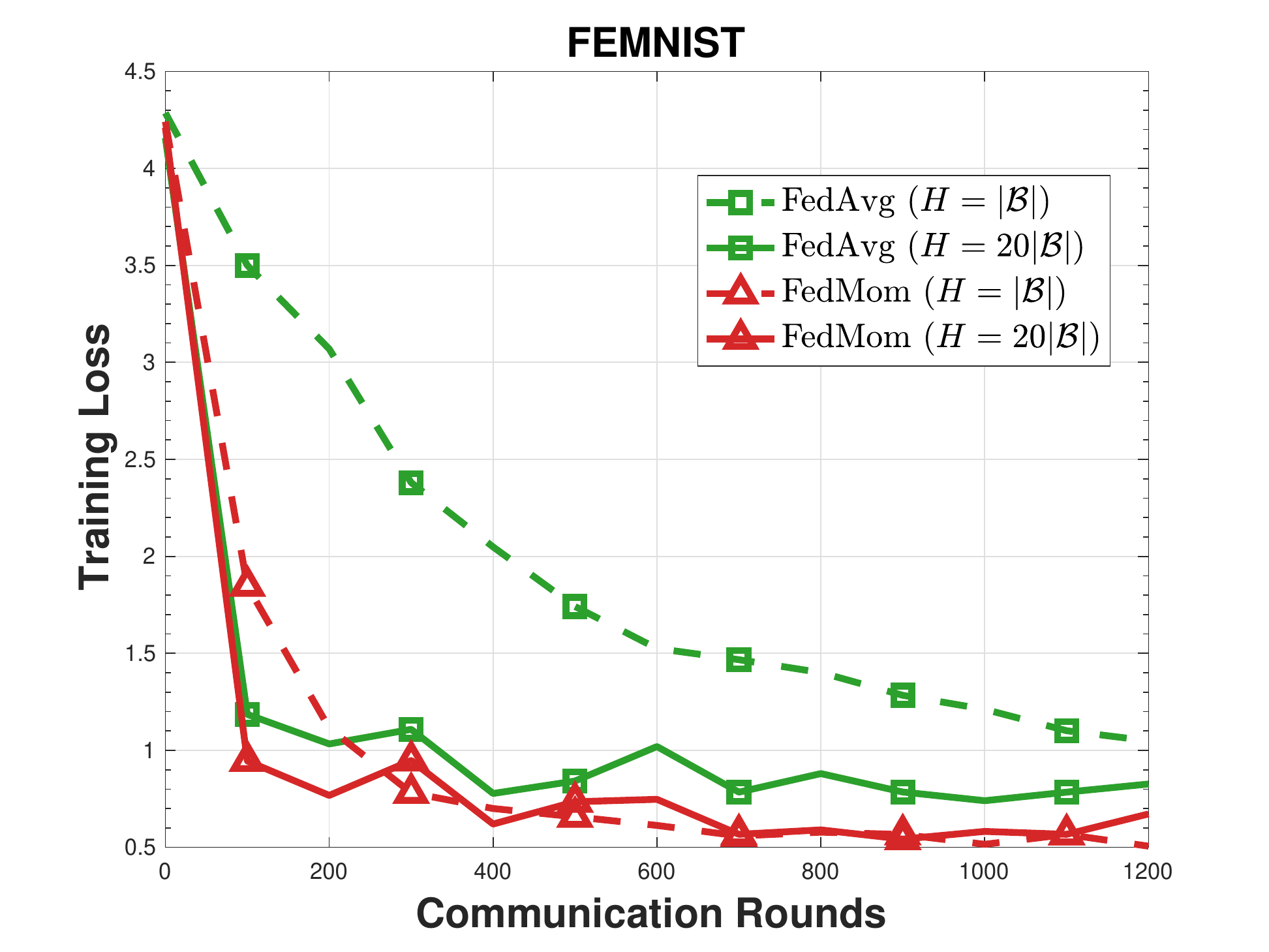}
	\end{subfigure}
	\caption{Training loss of FedMom and FedAvg methods when we vary the value of step size $\gamma$ and  local iterations $H$ on FEMNIST dataset. }
	\label{fig::fedmom_avg}
\end{figure*}

\subsection{Implementation Details}
Our implementations are based on the LEAF project \cite{caldas2018leaf}, which is a benchmark for federated learning algorithms. As in Table  \ref{table:data}, there are two tasks in the experiment: the digit recognition task on FEMNIST dataset \cite{caldas2018leaf} and the character prediction task on Shakespeare dataset \cite{mcmahan2016communication}.  For the digit recognition task, we use LeNet in the experiment \cite{lecun1998gradient};  for the task of character prediction, we train a character-level LSTM language model, which is 1-layer LSTM with $128$ nodes \cite{kim2016character}. To simulate the setting of federated learning, we set $M=2$ in all experiments, such that only two clients communicate with the server at each iteration. We let $\eta=\frac{K}{M}$ for two datasets; $|\mathcal{B}|$ represents the number of mini-batches in each epoch with batch size $B=10$.  For FedMom algorithm, we let $\beta = 0.9$ in all experiments.

\subsection{Direction of Biased Gradient }
\label{sec::exp_1}
We train neural networks using the FedAvg algorithm and visualize in Figure \ref{fig::sigma_femnist} the variations of $ \mathbb{E}_{S_t} \left<g_t, w_t - w^*  \right> $ during the course of optimization. Positive values denote that $g_t$ is heading towards the target solution. We approximate the expectation of $\left<g_t, w_t - w^*  \right>$ by taking the average of $\left<g_t, w_t - w^*  \right>$ every $100$ communication rounds. $w^*$ is set as $w_{2000}$, which is the model after $2000$ communication rounds. Taking the left figure on FEMNIST as an example, we have two observations.  First, the values are large at the beginning of optimization, which means the model is far from the target point at first and it moves towards the target point at a fast speed. After a number of rounds, the model is close to the target point and the value of $ \mathbb{E}_{S_t} \left<g_t, w_t - w^*  \right> $ becomes small.  Secondly, it is clear that the values of  $ \mathbb{E}_{S_t} \left<g_t, w_t - w^*  \right> $ are larger than $0$ most of the time.  We can also draw similar conclusions according to the result on Shakespeare dataset. Therefore, $g_t$ in FedAvg algorithm is an appropriate direction towards the target point, although it is biased.

\subsection{Investigating FedAvg and FedSGD}
\label{sec::exp_2}
In this section, we investigate why FedAvg converges faster than FedSGD empirically. We compare these two methods by training digit recognition task on FEMNIST dataset. In Figure \ref{fig::fedavg_sgd}, we visualize the difference  of  $ \mathbb{E}_{S_t}\left<g_t^{FedAvg}, w_t - w^*  \right>$ and $ \mathbb{E}_{S_t}\left< g_t^{FedSGD}, w_t - w^*  \right>$ during the course of optimization. In the leftmost figure, we can observe that the ``inner product'' of  FedAvg is larger than FedSGD all the time. At the same time, FedAvg converges faster than FedSGD regarding the training loss and testing accuracy. Experimental results indicate that FedAvg  is moving towards a better direction to the target point than FedSGD.

\subsection{ Convergence Comparison}
\label{sec::exp_3}
We compare the   convergence of FedSGD, FedAvg and FedMom, with results visualized in Figure \ref{fig::cmp_femnist}. There are two observations: ($i$) we know that FedAvg always converges faster than FedSGD by a large margin;  ($ii$) FedMom converges faster than FedAvg given similar step size $\gamma$ in all experiments.

In Figure \ref{fig::fedmom_avg}, we evaluate the proposed method by varying the step size $\gamma$ and local iterations $H$ on FEMNIST dataset. In the left figure, FedMom is always works better than FedAvg when we select a similar step size $\gamma$. Besides, it is clear that FedMom is more robust to the selection of step size $\gamma$. However, the performance of FedAvg with smaller $\gamma$ drops severely. When varying $H$, we observe similar results that FedMom performs more robust than FedAvg. Thus,  FedMom is a more practical method because it is easier to tune step size $\gamma$ and iterations $H$ than the compared methods.

\section{Conclusions}
We have investigated model averaging in the federated averaging algorithm, and have reformulated it as a gradient-based method with biased gradients. As a result, we derived the first convergence proof of the federated averaging algorithm for nonconvex problems. 
Based on our new perspective, we propose a novel federated momentum algorithm (FedMom) and prove that it is guaranteed to converge to critical solutions for non-convex problems. In the experiments, we compare FedMom with FedAvg and FedSGD by conducting simulated federated learning experiments on the digit recognition task and the character prediction task. Experimental results demonstrate that the proposed FedMom converges faster than the compared methods on both tasks and is easier to tune parameters as well.
More important, our research results open up new research directions for federated learning.

\bibliographystyle{ACM-Reference-Format}
\bibliography{acmart}

\begin{thebibliography}{10}

\bibitem{allen2014linear}
Z.~Allen-Zhu and L.~Orecchia.
\newblock Linear coupling: An ultimate unification of gradient and mirror
  descent.
\newblock {\em arXiv preprint arXiv:1407.1537}, 2014.

\bibitem{bonawitz2019towards}
K.~Bonawitz, H.~Eichner, W.~Grieskamp, D.~Huba, A.~Ingerman, V.~Ivanov,
  C.~Kiddon, J.~Konecny, S.~Mazzocchi, H.~B. McMahan, et~al.
\newblock Towards federated learning at scale: System design.
\newblock {\em arXiv preprint arXiv:1902.01046}, 2019.

\bibitem{bottou2016optimization}
L.~Bottou, F.~E. Curtis, and J.~Nocedal.
\newblock Optimization methods for large-scale machine learning.
\newblock {\em arXiv preprint arXiv:1606.04838}, 2016.

\bibitem{caldas2018expanding}
S.~Caldas, J.~Kone{\v{c}}ny, H.~B. McMahan, and A.~Talwalkar.
\newblock Expanding the reach of federated learning by reducing client resource
  requirements.
\newblock {\em arXiv preprint arXiv:1812.07210}, 2018.

\bibitem{caldas2018leaf}
S.~Caldas, P.~Wu, T.~Li, J.~Kone{\v{c}}n{\`y}, H.~B. McMahan, V.~Smith, and
  A.~Talwalkar.
\newblock Leaf: A benchmark for federated settings.
\newblock {\em arXiv preprint arXiv:1812.01097}, 2018.

\bibitem{chen2016revisiting}
J.~Chen, X.~Pan, R.~Monga, S.~Bengio, and R.~Jozefowicz.
\newblock Revisiting distributed synchronous sgd.
\newblock {\em arXiv preprint arXiv:1604.00981}, 2016.

\bibitem{cheng2019secureboost}
K.~Cheng, T.~Fan, Y.~Jin, Y.~Liu, T.~Chen, and Q.~Yang.
\newblock Secureboost: A lossless federated learning framework.
\newblock {\em arXiv preprint arXiv:1901.08755}, 2019.

\bibitem{dean2012large}
J.~Dean, G.~Corrado, R.~Monga, K.~Chen, M.~Devin, M.~Mao, A.~Senior, P.~Tucker,
  K.~Yang, Q.~V. Le, et~al.
\newblock Large scale distributed deep networks.
\newblock In {\em Advances in neural information processing systems}, pages
  1223--1231, 2012.

\bibitem{hard2018federated}
A.~Hard, K.~Rao, R.~Mathews, F.~Beaufays, S.~Augenstein, H.~Eichner, C.~Kiddon,
  and D.~Ramage.
\newblock Federated learning for mobile keyboard prediction.
\newblock {\em arXiv preprint arXiv:1811.03604}, 2018.

\bibitem{hinton2012neural}
G.~Hinton, N.~Srivastava, and K.~Swersky.
\newblock Neural networks for machine learning lecture 6a overview of
  mini-batch gradient descent.
\newblock {\em Cited on}, page~14, 2012.

\bibitem{jeong2018communication}
E.~Jeong, S.~Oh, H.~Kim, J.~Park, M.~Bennis, and S.-L. Kim.
\newblock Communication-efficient on-device machine learning: Federated
  distillation and augmentation under non-iid private data.
\newblock {\em arXiv preprint arXiv:1811.11479}, 2018.

\bibitem{kairouz2019advances}
P.~Kairouz, H.~B. McMahan, B.~Avent, A.~Bellet, M.~Bennis, A.~N. Bhagoji,
  K.~Bonawitz, Z.~Charles, G.~Cormode, R.~Cummings, et~al.
\newblock Advances and open problems in federated learning.
\newblock {\em arXiv preprint arXiv:1912.04977}, 2019.

\bibitem{kim2016character}
Y.~Kim, Y.~Jernite, D.~Sontag, and A.~M. Rush.
\newblock Character-aware neural language models.
\newblock In {\em AAAI}, pages 2741--2749, 2016.

\bibitem{kingma2014adam}
D.~Kingma and J.~Ba.
\newblock Adam: A method for stochastic optimization.
\newblock {\em arXiv preprint arXiv:1412.6980}, 2014.

\bibitem{konevcny2016federated}
J.~Kone{\v{c}}n{\`y}, H.~B. McMahan, F.~X. Yu, P.~Richt{\'a}rik, A.~T. Suresh,
  and D.~Bacon.
\newblock Federated learning: Strategies for improving communication
  efficiency.
\newblock {\em arXiv preprint arXiv:1610.05492}, 2016.

\bibitem{lecun1998gradient}
Y.~LeCun, L.~Bottou, Y.~Bengio, and P.~Haffner.
\newblock Gradient-based learning applied to document recognition.
\newblock {\em Proceedings of the IEEE}, 86(11):2278--2324, 1998.

\bibitem{lee2015distributed}
J.~D. Lee, Q.~Lin, T.~Ma, and T.~Yang.
\newblock Distributed stochastic variance reduced gradient methods and a lower
  bound for communication complexity.
\newblock {\em arXiv preprint arXiv:1507.07595}, 2015.

\bibitem{li2014efficient}
M.~Li, T.~Zhang, Y.~Chen, and A.~J. Smola.
\newblock Efficient mini-batch training for stochastic optimization.
\newblock In {\em Proceedings of the 20th ACM SIGKDD international conference
  on Knowledge discovery and data mining}, pages 661--670. ACM, 2014.

\bibitem{li2019convergence}
X.~Li, K.~Huang, W.~Yang, S.~Wang, and Z.~Zhang.
\newblock On the convergence of fedavg on non-iid data.
\newblock {\em arXiv preprint arXiv:1907.02189}, 2019.

\bibitem{lian2015asynchronous}
X.~Lian, Y.~Huang, Y.~Li, and J.~Liu.
\newblock Asynchronous parallel stochastic gradient for nonconvex optimization.
\newblock In {\em Advances in Neural Information Processing Systems}, pages
  2737--2745, 2015.

\bibitem{lin2018don}
T.~Lin, S.~U. Stich, and M.~Jaggi.
\newblock Don't use large mini-batches, use local sgd.
\newblock {\em arXiv preprint arXiv:1808.07217}, 2018.

\bibitem{loshchilov2017fixing}
I.~Loshchilov and F.~Hutter.
\newblock Fixing weight decay regularization in adam.
\newblock {\em arXiv preprint arXiv:1711.05101}, 2017.

\bibitem{ma2015adding}
C.~Ma, V.~Smith, M.~Jaggi, M.~I. Jordan, P.~Richt{\'a}rik, and
  M.~Tak{\'a}{\v{c}}.
\newblock Adding vs. averaging in distributed primal-dual optimization.
\newblock {\em arXiv preprint arXiv:1502.03508}, 2015.

\bibitem{mcmahan2016communication}
H.~B. McMahan, E.~Moore, D.~Ramage, S.~Hampson, et~al.
\newblock Communication-efficient learning of deep networks from decentralized
  data.
\newblock {\em arXiv preprint arXiv:1602.05629}, 2016.

\bibitem{mcmahan2017learning}
H.~B. McMahan, D.~Ramage, K.~Talwar, and L.~Zhang.
\newblock Learning differentially private recurrent language models.
\newblock {\em arXiv preprint arXiv:1710.06963}, 2017.

\bibitem{nesterov1983method}
Y.~Nesterov.
\newblock A method for unconstrained convex minimization problem with the rate
  of convergence o (1/k\^{} 2).
\newblock In {\em Doklady AN USSR}, volume 269, pages 543--547, 1983.

\bibitem{polyak1964some}
B.~T. Polyak.
\newblock Some methods of speeding up the convergence of iteration methods.
\newblock {\em USSR Computational Mathematics and Mathematical Physics},
  4(5):1--17, 1964.

\bibitem{qian1999momentum}
N.~Qian.
\newblock On the momentum term in gradient descent learning algorithms.
\newblock {\em Neural networks}, 12(1):145--151, 1999.

\bibitem{reddi2015variance}
S.~J. Reddi, A.~Hefny, S.~Sra, B.~Poczos, and A.~J. Smola.
\newblock On variance reduction in stochastic gradient descent and its
  asynchronous variants.
\newblock In {\em Advances in Neural Information Processing Systems}, pages
  2647--2655, 2015.

\bibitem{robbins1951stochastic}
H.~Robbins and S.~Monro.
\newblock A stochastic approximation method.
\newblock {\em The annals of mathematical statistics}, pages 400--407, 1951.

\bibitem{sahu2018convergence}
A.~K. Sahu, T.~Li, M.~Sanjabi, M.~Zaheer, A.~Talwalkar, and V.~Smith.
\newblock On the convergence of federated optimization in heterogeneous
  networks.
\newblock {\em arXiv preprint arXiv:1812.06127}, 2018.

\bibitem{smith2017federated}
V.~Smith, C.-K. Chiang, M.~Sanjabi, and A.~S. Talwalkar.
\newblock Federated multi-task learning.
\newblock In {\em Advances in Neural Information Processing Systems}, pages
  4424--4434, 2017.

\bibitem{stich2018local}
S.~U. Stich.
\newblock Local sgd converges fast and communicates little.
\newblock {\em arXiv preprint arXiv:1805.09767}, 2018.

\bibitem{wang2018cooperative}
J.~Wang and G.~Joshi.
\newblock Cooperative sgd: A unified framework for the design and analysis of
  communication-efficient sgd algorithms.
\newblock {\em arXiv preprint arXiv:1808.07576}, 2018.

\bibitem{yang2019federated}
Q.~Yang, Y.~Liu, T.~Chen, and Y.~Tong.
\newblock Federated machine learning: Concept and applications.
\newblock {\em ACM Transactions on Intelligent Systems and Technology (TIST)},
  10(2):12, 2019.

\bibitem{yang2016unified}
T.~Yang, Q.~Lin, and Z.~Li.
\newblock Unified convergence analysis of stochastic momentum methods for
  convex and non-convex optimization.
\newblock {\em arXiv preprint arXiv:1604.03257}, 2016.

\bibitem{you2017large}
Y.~You, I.~Gitman, and B.~Ginsburg.
\newblock Large batch training of convolutional networks.
\newblock {\em arXiv preprint arXiv:1708.03888}, 2017.

\bibitem{yu2018parallel}
H.~Yu, S.~Yang, and S.~Zhu.
\newblock Parallel restarted sgd for non-convex optimization with faster
  convergence and less communication.
\newblock {\em arXiv preprint arXiv:1807.06629}, 2018.

\bibitem{zhang2015deep}
S.~Zhang, A.~E. Choromanska, and Y.~LeCun.
\newblock Deep learning with elastic averaging sgd.
\newblock In {\em Advances in Neural Information Processing Systems}, pages
  685--693, 2015.

\bibitem{zhang2015disco}
Y.~Zhang and X.~Lin.
\newblock Disco: Distributed optimization for self-concordant empirical loss.
\newblock In {\em International conference on machine learning}, pages
  362--370, 2015.

\bibitem{zhou2017convergence}
F.~Zhou and G.~Cong.
\newblock On the convergence properties of a $ k $-step averaging stochastic
  gradient descent algorithm for nonconvex optimization.
\newblock {\em arXiv preprint arXiv:1708.01012}, 2017.

\bibitem{5g}
M.~Zhou.
\newblock 5g will mean big boom for smart devices, ericsson says.
\newblock 2018.

\end{thebibliography}

\appendix

\end{document}